\documentclass{article}

\usepackage[final]{neurips_2019}
\usepackage{graphicx}
\usepackage{subcaption}
\usepackage[utf8]{inputenc} 
\usepackage[T1]{fontenc}    
\usepackage{hyperref}       
\usepackage{url}            
\usepackage{booktabs}       
\usepackage{amsfonts,amsmath,amsthm,amssymb, mathrsfs}       
\usepackage{nicefrac}       
\usepackage{microtype}      
\usepackage{color}

\newtheorem{definition}{Definition}

\newtheorem{lemma}{Lemma}
\newtheorem{theorem}{Theorem}

\usepackage[boxed]{algorithm}
\usepackage{graphicx}
\usepackage{capt-of}
 \usepackage{algorithmicx}
 \usepackage{algpseudocode}
 \renewcommand{\algorithmicrequire}{\textbf{Input:}}
 \renewcommand{\algorithmicensure}{\textbf{Output:}}

\newcommand{\h}{\mathfrak{h}}
\newcommand{\g}{\mathfrak{g}}

\newcommand\Mark[1]{\textsuperscript#1}

\title{
 Differentially Private Distributed Data Summarization under Covariate Shift \thanks{\Mark{1}Equal contribution by these authors.}}

\author{%
  Kanthi K. Sarpatwar \Mark{1} \\
IBM Research\\
  \texttt{sarpatwa@us.ibm.com}\\
  \And
   Karthikeyan Shanmugam \Mark{1} \\
IBM Research AI\\
  \texttt{karthikeyan.shanmugam2@ibm.com}\\
  \And 
  Venkata Sitaramagiridharganesh Ganapavarapu \\
IBM Research\\
  \texttt{giridhar.ganapavarapu@ibm.com}\\
  \And
    Ashish Jagmohan \\
IBM Research\\
  \texttt{ashishja@us.ibm.com}\\
    \And
    Roman Vaculin \\
IBM Research\\
  \texttt{vaculin@us.ibm.com}\\
}

\begin{document}
\maketitle
\begin{abstract}

We envision Artificial Intelligence marketplaces to be platforms where consumers, with very less data for a target task, can obtain a relevant model by accessing many private data sources with vast number of data samples.  One of the key challenges is to construct a training dataset that matches a target task without compromising on privacy of the data sources. To this end, we consider the following distributed data summarizataion problem. Given K private source datasets denoted by $[D_i]_{i\in [K]}$ and a small target validation set $D_v$, which may involve a considerable covariate shift with respect to the sources, compute a summary dataset $D_s\subseteq \bigcup_{i\in [K]} D_i$ such that its statistical distance from the validation dataset $D_v$ is minimized. We use the popular Maximum Mean Discrepancy as the measure of statistical distance. The non-private problem has received considerable attention in prior art, for example in prototype selection (Kim et al., NIPS 2016). Our work is the first to obtain strong differential privacy guarantees while ensuring the quality guarantees of the non-private version. We study this problem in a Parsimonious Curator Privacy Model, where a trusted curator coordinates the summarization process while minimizing the amount of private information accessed. Our central result is a novel protocol that (a) ensures the curator accesses at most $O(K^{\frac{1}{3}}|D_s| + |D_v|)$ points (b) has formal privacy guarantees on the leakage of information between the data owners and (c) closely matches the  best known non-private greedy algorithm. Our protocol uses two hash functions, one inspired by the Rahimi-Recht random features method and the second leverages state of the art differential privacy mechanisms. Further, we introduce a novel ``noiseless'' differentially private auctioning protocol for winner notification, which may be of independent interest.  Apart from theoretical guarantees, we demonstrate the efficacy of our protocol using real-world datasets.

\end{abstract}
\section {Introduction}

Integrating new types of data to drive analytics based decision-making can contribute significant  economic impact across a broad spectrum of industries including healthcare, banking, insurance, travel, and urban planning.  This has led to the emergence of complex data ecosystems consisting of heterogeneous (and overlapping) data generators, aggregators, and analytics providers. In general, participants in these ecosystems are looking to monetize a class of assets that we term \textit{AI assets}; such assets include raw and aggregated data, as well as models trained on such data. A recent Mckinsey global survey found, for example, that more than half of the respondents in sectors including basic materials and energy, financial services, and high tech stated that their companies had begun to monetize their data assets [\cite{mckinsey2}].

In light of the above, in this work we consider the basic setting of an AI Marketplace : a consumer arrives with a small dataset, referred to as a ``validation'' dataset, and wants to build a prediction model that performs well on this dataset. However, the model training process requires huge amount of data, that it must acquire from multiple private sources. Fundamentally, the AI Marketplace must address a {\em transfer learning problem}, where the distribution of data at different sources is considerably different from each other and even from the validation dataset. 
The Marketplace must facilitate  transactions of data points from multiple sources towards the consumer's task by forming a training dataset that is close in some distance measure to the validation dataset. In the process, it must preserve data ownership and privacy as much as possible.

Consider the following scenario in the health care domain, as an example. Suppose the consumer is a newly established cancer hospital and the data sources are cancer institutions from different geographical locations across the globe. The goal of the new hospital is to construct ML models that, say, can predict early onset of some form of cancer. The quality of the model depends on the demography of its patients and therefore it is crucial to collect data that matches a small validation set that is representative of the demography. The individual sources clearly have widely different demographic data. The goal of an AI Marketplace is to enable private collection of a dataset sampled from these sources that matches the demography of the new institute. From the privacy perspective, there are two desirable properties: (a) The multiple data owners are typically ``competitors'' and therefore, individual data must be protected (for e.g. in a differentially private manner) from each other. (b) The platform (we use the term curator) must be ``parsimonious'' in handling data, i.e., it should access information on a ``need to know'' basis.

Motivated by the above, we consider the following problem. We consider K data owners with private datasets \( D_1, \ldots, D_K \), and a data consumer who wishes to build a model for a specific task. The specific task is embodied by the consumer possessing a validation dataset \( D_v \). The data consumer would like to procure a subset of data from each private dataset which is well-matched to its task. A parsimonious trusted curator does the collection of points. We call this the parsimonious curator model. Although trusted, we wish to minimize the number of points accessed by the curator to construct the the final summary.  In turn, the \( K \) data owners would like to ensure that their data is private with respect to other data owners. The curator exchanges messages and data points with the data owners. We seek to make the exchanges by the curator to the data owners differentially private.

\textbf{Our Contributions:} We propose a novel protocol, based on an iterative hash-exchange mechanism, that enables the curator to construct a summary data set \( D_s \) from the K owner datasets. A central result of the paper shows that the proposed protocol simultaneously satisfies the following desired properties: (i) The constructed dataset \( D_s \) is well-matched to the validation dataset \( D_v \), in terms of having a small Maximum Mean Discrepancy (MMD) (\cite{mmd}); (ii) The protocol exchanges with any data owner $i$ is $( \epsilon, \delta)$-differentially private with respect to the other owner datasets $\cup_{j \neq i} D_j$; and (iii) The parsimonious curator accesses at most $O(K^\frac{1}{3}|D_s| + |D_v|)$ data points. Qualitatively, we expect the protocol to produce data summaries that are useful for model building while maintaining differential privacy. We show through empirical evaluation that this is indeed the case; by examining generalization error on two example tasks, we show that the protocol pays only a small price for its differential privacy guarantees.


\textbf{Prior Work:}
 \textit{Privacy Preserving Learning Algorithms:} There is a long line of work that  considers private empirical risk minimization that seeks to optimize the trade-off between accuracy of a trained classifier and differential privacy guarantees with respect to the training set [\cite{kasiviswanathan2011can,chaudhuri2011differentially,song2013stochastic,kifer2012private,bassily2014private,shokri2015privacy,hamm2016learning,wu2016differentially,abadi2016deep,pathak2010multiparty,thakurta2013differentially,rubinstein2012learning,talwar2015nearly,dwork2014using}]. One of the most notable in this line of work is the idea of adding noise to stochastic gradient iterations to preserve privacy [\cite{song2013stochastic,abadi2016deep,shokri2015privacy}]. We do not consider the problem of learning a classifier directly. Our goal is to summarize diverse data sources in a distributed private manner to match a given validation set in a transfer learning setting. 
All the above works of privacy preserving learning algorithms can be applied after our summarization step. In  \cite{rubinstein2012learning,chaudhuri2011differentially}, authors use noisy Rahimi-Recht Fourier features to release a representation of the support vectors for differentially private SVM classifier release. Our purpose of using Rahimi-Recht Fourier features is different and is used to expose the partial MMD objective at every round and in conjunction with novel private auctioning mechanisms.
 
 \textit{Privately Aggregating Teacher Ensembles:} Several works have considered the following setting: An ensemble of teacher classifiers, each trained on private data sources, noisily predict labels on an unlabelled public dataset that is further used to train a student model [\cite{papernot2016semi}]. Again, this is different from our transfer learning setting where the various distributions are matched to a target task first handling covariate shift.

 Another related line of work is differentially private submodular optimization [\cite{mitrovic2017differentially}]. While they consider a single private source, we handle multiple private data sources in optimizing a specific statistical distance ($MMD$). 
 Our techniques leverage state of the art methods on privacy preserving mechanisms found in \cite{dwork2014algorithmic,hardt2012,hardt2010multiplicative}.
 
 \textit{Domain Adaptation Methods:} For the transfer learning problem, the existing domain adaptation methods [\cite{ganin2014unsupervised,tzeng2014deep}] ensure the following: they learn a representation $\phi(\mathbf{x})$ such that $\phi(\cdot)$ of the source and the target are similar in distance (MMD metric has been used to regularize the distance penalty) and that classifying based on $\phi(\cdot)$ on the source have very high accuracy. However, most existing approaches used differentiable models like deep learning to achieve this - to learn $\phi(\cdot)$. In our methods, we first match the distributions in the ambient space by sub-selecting points and then train any suitable classifier. One advantage is that we can train any classifier after the moment matching step (Xgboost, Decision Tree, SVMs etc.). 
 If one wants to make an existing domain adaptation algorithm private with respect to any pair of participants, one has to add noise to gradients computed at every step. The state of the art in differential privacy for deep learning [\cite{abadi2016deep}] (in the non-transfer learning setting) adds Gaussian noise whose variance is linear in the number of iterations per step which significantly degrades the performance. In our method, we gain on this aspect as we add noise per point acquisition. 
 
 \textit{Federated Learning:} We also note there is a  distinction between our transfer learning setting from that of Federated learning \cite{mcmahan2016communication}. The validation set distribution
is distinct from each of the individual data source distributions. There are significant covariate shifts between these. Federated learning would assume a training distribution which is obtained by sampling from different data sources uniformly at random or with a specific mixture distribution. In fact, in our experiments we contrast with training done on uniform samples which is a proxy for federated learning.

\section{Problem Setting}
The setting has $K$ data owners with private datasets denoted by $D_1, D_2, \ldots D_K$. Here, $D_i \in \mathbb{R}^{m_i \times n}$ where $m_i$ denotes the number of points and $n$ denotes the their dimension. Further, there exists a ``consumer'' entity that wants to form a summary dataset (which can be used for downstream training goals) $D_s \subseteq \bigcup_i D_i$ and $\lvert D_s \rvert = p$. The quality of the summary set is measured by its {\em closeness} to a target validation dataset $D_v \in \mathbb{R}^{m \times n}$ which is private to the consumer. We measure the closeness of $D_s$ to $D_v$ using the MMD (Maximum Mean Discrepancy) statistical distance defined below.


{\bf Definition:} The sample MMD distance for finite datasets $D \in \mathbb{R}^{m_1 \times n}$ and $D' \in \mathbb{R}^{m_2 \times n}$ is given by:
\begin{align}
  &\mathrm{MMD}^2(D,D') = \frac{1}{m_1^2}\sum_{x,x' \in D} k(x,x')  -\frac{2}{m_1m_2} \sum_{x \in D,y \in D'} k(x,y) +
  \frac{1}{m_2^2} \sum_{y,y' \in D'} k(y,y') 
\end{align}
 where $k(\cdot,\cdot)$ is a kernel function underlying an RKHS (Reproducing Kernel Hilbert Space) function space such that $k(x,y)=k(y,x)$ and $k(\cdot,\cdot)$ is positive definite. 

 \emph{Differential Privacy:}  
We adopt the following definition of differential privacy [\cite{dwork2006our}] in our work. On a high level, it means that two datasets that differ in at most one point should not cause a differentially private algorithm to produce output that are very different statistically. Formally,

{\bf Definition:}  The output of a randomized algorithm ${\mathcal A}(D)$ is $(\epsilon,\delta)$ differentially private with respect to the input dataset $D$ if for any two neighboring datasets $D,D'$  that differs in one data point, 
  \begin{align}
  P \left({\mathcal A}(D) \in E \right) \leq e^{\epsilon} P \left({\mathcal A}(D') \in E \right) + \delta.
 \end{align} 
 for all events $E$ that can be defined on the output space.    

\emph{Parsimonious Curator Privacy Model:} 
We assume that there exists a trusted curator, called {\em aggregator}, that collects the summary data points $D_s$. The participants holding data $D_i$ wish to preserve the privacy of their individual data points.  The model satisfies the following constraints:a) During the protocol run, the curator must not have access to more than $\rho(\lvert D_s \rvert + \lvert D_v \rvert) $ points. We refer to such a protocol as {\bf $\rho$-parsimonious} protocol. 
The aggregator needs to collect points that closely match $D_v$ in MMD distance. Therefore the aggregator at least sees $\lvert D_s \rvert$ points in this framework. This forms a natural $\lvert D_s \rvert + \lvert D_v \rvert$ lower bound on how many points the aggregator has to access. Therefore, we define a $\rho$-parsimonious aggregator who sees $\rho$ times the minimum required. 
b)  Communication to a non-trusted participant $i$ is differentially private with respect to all other datasets i.e. $ \cup_{j \neq i} D_i \cup D_v$. 
This setting can be viewed as an intermediate regime between the  ``centralized setting" and ``localized setting'' [\cite{nissim2017clustering}] considered in the prior works. In other words, the source $D_i$ knowing all but one point in the union of other datasets as side information must not know much (in the differential privacy sense) about the missing point given all the communication to it during the protocol (standard informed adversary model with respect to union of other datasets $\bigcup \limits_{j \neq i} D_j$). Preservation of differential privacy across data sources constrains the aggregator to collect more points than necessary (i.e. $\lvert D_s\rvert + \lvert D_v \rvert$).

{\em Main Problem:} Is there a  $(\epsilon,
\delta)$ differentially private protocol in the parsimonious curator model, that outputs a subset $D_s \subseteq \bigcup_i D_i: \lvert D_s \rvert =p $ that (approximately) minimizes $\mathbb{E}[\mathrm{MMD}^2(D_s,D_v)] $ ?  

\textit{Incentives:} The aggregator needs to train a downstream task on a test
distribution that is similar to $D_v$. To this end, $\lvert D_s \rvert$ points (much larger is size than $D_v$) are being collected for training. In fact, one could think of the aggregator paying for the points. Our protocol is approximately the best way to obtain such points. There is no incentive for the aggregator to cheat since it has to pay for the collected points. The data providers are happy to provide a set as long as they are compensated and
other data sources do not know about their data (in a differential privacy sense.).

Every data source would be able to monetize their
contribution in proportion to the value they provide to the summary. After the protocol ends, value of a data source’s
contribution could be deemed proportional to the sum of winning marginal bids from the source. Value attribution
based on this would be a incentive for data holders to participate. We address the problem of value attribution to data sources in a companion paper (\cite{sarpatwar2019blockchain}). We only focus on the privacy and parsimonious constraints.

{\bf Our Approach:}
We briefly summarize the greedy approach to solve the moment matching problem without privacy constraints. Our fundamental contribution is to make it differentially private in the parsimonious curator model.

\emph{Greedy Algorithm Without Privacy:}
Our objective is to form a summary $D_s$ of size $p$ by collecting points from all the data owners. We maximize the following normalized MMD objective [\cite{kim2016examples}] as described below. For fixed validation set $D_v$ such that $|D_v| = m$ and the summary set $D_s$, the objective $J(D_s)$ is as follows:
\begin{align}\label{submod}
J(D_s) =&  \sum_{i,j\in D_v} \frac{k(y_i,y_j)}{m^2} - \mathrm{MMD}^2(D_v,D_s) =\sum_{i\in D_v, j\in D_s} \frac{2k(y_i,x_j)}{m|D_s|} -  \sum_{i,j\in D_s} \frac{k(x_i,x_j)}{|D_s|^2} 
\end{align}

Note that our objective here is different from the one used in~\cite{kim2016examples}, in that we do not have the property $S \subseteq V$. Submodularity of this function does not follow from their work directly. In Section~\ref{app:submod} of appendix, we show that the function is submodular under some condition on the kernel function. This condition is satisfied if the distance between any two points is $\Omega(\sqrt{\log N})$ and when the RBF kernel $k(\mathbf{x},\mathbf{y})=\exp(-\gamma \lVert \mathbf{x} -\mathbf{y} \rVert_2^2)$ is used with some constant $\gamma>0$. 


\begin{theorem}
\label{thm:submod}
Let $N$ be the total number of points in the system. Given a diagonally dominant kernel matrix $\mathbf{K} \in \mathbb{R}^{N\times N}$ satisfying $k_{i,i} = k^*$, for any $i\in [N]$ and $k_{i,j} \leq \frac{k^*}{N^3+3N^2+N}$ for any $i\neq j$, then $J(S)$ is a non-negative, monotone and submodular function. 
\end{theorem}

It has been proven that the following iterative greedy approach yields a constant factor approximation guarantee, given that the objective is a non-negative monotone submodular [\cite{nemhauser1978analysis}] function. Iteratively, until the required summary size is achieved:
(a) each participant computes its marginally best point $y$, i.e., that maximizes $J(D_s+y)- J(D_s) $ and (b) curator collects the marginally best points from various participants and adds the best among them to the summary.





{\em Our Private Algorithm:}
 The focus of the paper is to adapt this greedy approach with \textit{privacy guarantees} in the parsimonious curator model. In our private protocol, the curator collects the data points in $D_s$ in a greedy fashion as above. 
 However, there is a key challenge on the privacy front: 

\emph{Challenges:} During the implementation of the greedy algorithm, the curator maintains a set of points $D_s= \{x_1, \ldots, x_k \}$. To calculate the marginal gain with respect to Equation (\ref{submod}), we observe that the curator needs to expose a function of the form $\sum \alpha_i k(x_i,\cdot)$ to every participant for some constants $\alpha_i$ (this will become clear later). However, sharing the points in the raw form would be a violation of privacy constraints at the participants. Further, over the course of multiple releases, any participant must not be able to acquire any information about previous data points of other participants. Therefore, the key issue is that the releases of the curator must be differentially private while enabling the computation of the (non-linear) marginal gain, over all the iterations of the protocol. Beyond enabling the computation of ``best'' points, privacy concerns also arise in the actual collection of data points. Indeed, even a private declaration of ``winners'' to data providers would result in the leakage of information on the quality of other data providers. 

\emph{Our Solution:} To solve these issues, we use two hash functions:

(a) $\h_1(\cdot)$ based on the random Fourier features method of Rahimi-Recht to hash every data point at the curator. This hash function is common to all the entities (i.e., curator and the participants) and satisfies the property that $\h_1(x)^T \h_1(y) \approx k(x,y)$ w.h.p., which is useful to convert the non-linear kernel computation (Equation~(\ref{submod})) to a linear one.  This enables approximate kernel computation by an entity external to the curator. Therefore, any entity can compute the marginal gain of a new point $y$ by $\sum \alpha_i k(x_i,y) \approx \sum \alpha_i \h_1(x_i)^T \h_1(y).$ Thus, the curator needs to only share $\sum \alpha_i \h_1(x_i)$. 

(b) a second hash function $\h_2(\cdot)$, whose randomness is private to the curator such that, $\h_2(\sum \alpha_i \h_1(x_i))^T \h_1(y) \approx \sum \alpha_i k(x_i,y) $ and $\h_2$ is differentially private with respect to $\h_1(x_i)$. A specific participant can observe multiple releases of $\sum \alpha_i \h_1(x_i)$ and potentially find out the last point that was added. Therefore, the releases of the sum vector $\sum \alpha_i \h_1(x_i)$ needs to be protected. 

Our $\h_2(\cdot)$ is a novel adaptation of the well-known MWEM method [\cite{hardt2012}]. 
The key technical challenge is to match the performance of the greedy algorithm while ensuring privacy properties of  $\h_2(\cdot)$  in order to protect data releases from the curator. Further, to address the privacy concerns in parsimonious data collection, we obtain a novel private auction mechanism that is $O(K^{\frac{1}{3}})$-parsimonious and $(\epsilon,\delta)$-differentially private, with no further loss in optimality. 
Aside from theory, we provide insights to make our protocol well-suited for practice and demonstrate its efficacy on real world datasets. 
\section{The Protocol}\label{sec:protocol}
Our protocol uses two different hash functions that we refer to as $\h_1(.)$ and $\h_2(.)$. The hash function $\h_1(.)$ is shared between the various data owners and the aggregator. The hash function $\h_2(.)$ is used by the aggregator to hash the current summary dataset before being broadcast to various participating entities (owners). 
We now describe both the hash functions $\h_1(\cdot),\h_2(\cdot)$.

\textbf{The Hash Function $\h_1(\cdot)$:}
Our first hash function, which is shared and used by various data owners and the aggregator is based on a well known distance preserving hash function formulated by \cite{rahimi2008random}. Formally, the hash function is defined in Algorithm \ref{algm:h1}. The main purpose of this hash function is to ensure that $\h_1(\mathbf{x})^T \h_1(\mathbf{y}) \approx k(\lVert \mathbf{x} - \mathbf{y} \rVert)$.
We assume an RBF kernel function throughout the paper which is given by $k(\Delta)= \exp(-\gamma \Delta^2 )$.  In Algorithm \ref{algm:h1}, $p(\omega)$ is the distribution defined by the Fourier transform of the kernel $k(\Delta)$, i.e. $p(\omega) = \frac{1}{2\pi}\int e^{-j\omega^T\Delta}k(\Delta) d\Delta$. Due to the RBF kernel, $p(\omega) = {\cal N}(0,2\gamma \mathbf{I}_n)$. The randomness in the hash function is due to $d$ random points drawn from this distribution as in Algorithm \ref{algm:h1}.


\begin{algorithm}
\caption{Computing the hash function $\h_1(\cdot)$.\label{algm:h1}}
\begin{algorithmic}[1]
\State {\bfseries Input:} {Point $\mathbf{x} \in \mathbb{R}^n$, parameter $\gamma$, dimension parameter $d$}
\State {\bfseries Output:}{ $\h_1(\mathbf{x})$}
\State Draw $\{\omega_i\}_{i=1}^d$ i.i.d from the same distribution 
$p(\omega) = {\cal N}(0,2\gamma \mathbf{I}_n)$ only \textbf{once} at the beginning of the protocol and reuse it over subsequent calls to $\h_1(\cdot)$. 
\State Draw samples $\{b_i\}_{i\in [d]}$ i.i.d uniformly from $[0,2\pi]$ only \textbf{once} at the beginning of the protocol. 
\State \Return $\h_1(\mathbf{x}) = \sqrt{\frac{2}{d}} \left[\cos (\omega_1^T\mathbf{x}+b_1), \cos (\omega_2^T\mathbf{x}+b_2) +  \ldots \cos (\omega_d^T\mathbf{x}+b_d)\right]^T$

\end{algorithmic}
\end{algorithm}

\textbf{The Hash Function $\h_2(\cdot)$:} Consider a dataset $D \in \mathbb{R}^{q \times d}$ consisting of vectors $\{\mathbf{v}_1, \mathbf{v}_2 \ldots \mathbf{v}_q \}$ such that $\mathbf{v}_i \in \mathbb{R}^{1 \times d}$ and $-\sqrt{\frac{2}{d}} \leq v_{ij} \leq \sqrt{\frac{2}{d}}, ~ 1 \leq i \leq q, ~1 \leq j \leq d$.
The hash function $\h_2(D)$ approximately computes the vector sum $ w(D)=\sum_{i} \mathbf{v}_i$ in a differentially private manner. Let $w(D,j)= \sum_{i} \mathbf{v}_{ij}$. We now provide the description of the $\h_2(\cdot)$ in Algorithm \ref{algm:h2}. The algorithm has two components: (a) The algorithm first quantizes the $q$ vectors in $D$ to obtain $D_Q$ such that the quantized coordinate values are from a grid $S$ of points $S=\{-1, -1+\eta, -1+2\eta ...\ldots 1-\eta, 1\}$, for a parameter $\eta$ (refer to Line~\ref{proc:quantize} in Algorithm~\ref{algm:h2}). (b) Then a random distribution $P_{\mathrm{avg}}$ over the space of all possible quantized vectors $S^{1 \times d}$ is found such that the expected vector under this distribution is close to the sum of the quantized vectors in $D_Q$. Further, the releases are also differential private. This second part relies on the MWEM mechanism of \cite{hardt2012}.

\emph{Full Algorithmic Description of $\h_2(\cdot)$ :} Let $\tilde{v}_1, \tilde{v}_2 \ldots \tilde{v}_q \in S^{1 \times d}$ be the quantized vectors in $D_Q$ and $w(D_Q,i)= \sum_{j=1}^q \tilde{v}_{ji}$.
Now, we will define probability mass functions $P_t(\mathbf{s} \in S^{1 \times d})$ for every time $t$ over the finite set $S^{1 \times d}$ whose cardinality is $|S|^{d}$. $P_t$ will be dependent only on $P_{t-1}$. We will define the distribution iteratively over $t \leq T$ iterations. Define $w(P,i)= q(\sum_{s \in S}s P_i(s))$ with respect to a probability mass function $P$ on $S^{1 \times d}$ where $P_i(s)$ is the marginal pmf on the $i$-th coordinate. The way $P_t$ is computed is given in  Algorithm \ref{algm:h2} (Steps $6$-$7$).


\begin{algorithm}
\caption{Computing the hash function $\h_2(\cdot)$.\label{algm:h2}}
\begin{algorithmic}[1]
\State {\bfseries Input:} {Dataset $D$, parameters $\varepsilon$, $\eta$ and $T$}
\State {\bfseries Output:}{ $\h_2(D, T, \varepsilon)$}
\State \label{line1} Obtain $D_Q \leftarrow \textsc{Quantization}(D, \eta)$. Let $P_0$ be the uniform distribution over the set $S^{1 \times d}$ where $S=\{-1, -1+\eta, -1+2\eta ...\ldots 1-\eta, 1\}$. 
\ForAll{$t \leq [T]$}
\State \label{line2} Sample a coordinate $i\in [d]$ with probability proportional to $\exp \left(\varepsilon \psi_{i}(D_Q) \right)$ where the score function: $\psi_{i}(D_Q) = |w(P_{t-1},i) - w(D_Q,i)|$.
Let the sampled coordinate be $i(t)$
\State \label{line3} Let $\mu_{i(t)} \leftarrow w(D_Q,i(t)) + Lap(1/\varepsilon)$. Compute the distribution satisfying $P_t(\mathbf{s}) \propto P_{t-1}(\mathbf{s}) \exp{[s_{i(t)}(\mu_{i(t)} - w(P_{t-1},i(t))~)/2q]}$
\EndFor 
\State $P_{\mathrm{avg}} = { \frac{1}{T}\sum_{t\in[T]} P_t}$.
\State {\bf return} $\sqrt{\frac{2}{d}}\frac{1}{q} \left[w(P_{\mathrm{avg}},1) \ldots w(P_{\mathrm{avg}},d)\right]= \h_2(D,\varepsilon)$ $=\h_2(D_Q,\varepsilon)$ 
\Statex \hrulefill
\Procedure{Quantization}{$D, \eta$}
\label{proc:quantize}
\State Define $Q(x) = \left \lbrace
\begin{array}{cc}
-1+k \eta, & w.p. \frac{(k+1) \eta -1-x }{\eta} \\
-1+ (k+1)\eta & w.p. \frac{x + 1- k \epsilon}{\eta}
\end{array} \right\rbrace$
where $k=\lfloor (x+1)/\eta \rfloor$.
Let $Q(\mathbf{v} = (v_1, v_2,\ldots, v_d)) =  (Q(v_1), Q(v_2), \ldots, Q(v_d))$
\State \Return $D_Q= \left\lbrace  Q\left( \sqrt{\frac{d}{2}} \mathbf{v}_i  \right) \right \rbrace_{i=1}^{q}$, 

\EndProcedure

\end{algorithmic}
\end{algorithm}

{\bf Description of the Protocol:}
We now describe our protocol in Algorithm \ref{algm:protocol} and the protocol parameters $\epsilon_v, \epsilon_{\ell,T}$ used. The protocol ensures two properties at the data owner:

\emph{Approximate Marginal Gain Computation:} The trusted aggregator at the beginning (Step \ref{alg3:line1}) shares $\tilde{\g} = \h_2(\h_1(D_v))$. We show that $\lVert\h_2( \h_1(D_v) ) - \sum_{\mathbf{x} \in D_v} \h_1(\mathbf{x}) \rVert_\infty$ is very small. Therefore, $\mathbf{\tilde{g}}^T \h_1(\mathbf{y})$ when computed at a data owner with a new point $\mathbf{y}$ approximates $\sum_{\mathbf{x} \in D_v} \h_1(\mathbf{x})^T \h_1(\mathbf{y})$. Similarly, over any other iteration $\ell$ (in Step $4$), the hashed vector $\g_{\ell}$ is such that $\g_{\ell}^T \h_1(\mathbf{y}) \approx \sum_{\mathbf{x} \in D_s} \h_1(\mathbf{x})^T \h_1(\mathbf{y})$. Since, $\h_1$ has the property that $\h_1(\mathbf{x})^T \h_1(\mathbf{y}) \approx k(\lVert \mathbf{x} - \mathbf{y}\rVert)$, we can ensure that the maximization in Step $13$ is approximately the marginal gain computation $J(D_s+\mathbf{y})- J(D_s)$.

\emph{Differential Privacy:} We also show that, due to application of $\h_2$, all the releases seen by any data owner $i$ are differentially private with respect to the current summary which also implies it is differentially private with respect to $\cup_{j \neq i} D_i - D_i$. 
Another key ingredient in our proof is in showing that the novel scheme in making the bid collection and winner notification process differentially private, while ensuring the parsimonious nature of the aggregator. 
Consider the Step~\ref{winbid} in Algorithm~\ref{algm:protocol}. Upon making a decision on the winning bid, the aggregator needs to acquire the winning point from the winner data source. Consider the following two naive ways of doing this: (a) Aggregator notifies the winner alone about the decision and acquires the data point.
    (b) Aggregator acquires data points from all the data sources. Keeping only the winner point, it discards the other points.
An important observation here is that the first alternative is not differentially private. Indeed, it leaks information about the data points of the participating data sources.  The second way is differentially private, indeed, each data source learns nothing new about other data sources. However, it is highly wasteful and contradicts the parsimonious nature of the aggregator. Indeed, in forming a summary of size $p$, it collects $Kp$ data points. Our novel private auction (Steps~\ref{proc:priv}-\ref{proc:priv:end} of Algorithm~\ref{algm:protocol})  obtains best of both scenarios, i.e., it is differentially private and accesses at most $O(pK^{\frac{1}{3}})$ data points in total. 
\begin{algorithm}
\caption{Description of the protocol\label{algm:protocol}.}
\begin{algorithmic}[1]
\State {\bfseries Input:} {$D_i$ $i\in [K]$, validation dataset $D_v$,  seed set $D_{\mathrm{init}}$,  params $\{\epsilon_{auc}, \epsilon_v, \{\epsilon_{\ell,T}\}_{\ell =1}^p, \tau \}$}.
\State {\bfseries Output:}{ Summary $D_s$: $D_s\subseteq \cup_{i\in [K]}D_i $ 
such that $|D_s| = p$}.
\State \emph{Aggregator} initializes  summary  $D_s \leftarrow D_{\mathrm{init}} $ and  broadcasts $\tilde{\g}=\h_2(\h_1(D_v),\epsilon_v)$.
\For{$\ell=1 \ldots p$} \label{alg3:line1}
\label{step:epoch}
\State \label{alg3:line2} \emph{Aggregator} broadcasts $\g_{\ell} = \h_2(\h_1(D_s),\epsilon_{\ell,T})$.
\State \label{alg3:line3} Each \emph{Data owner} $i\in [n]$ computes its ``bid'':
   $b_i = \max_{x\in D_i}\g_{\ell}^T \h_1(x) - \tilde{\g}^T \h_1(x) \frac{\ell}{\ell+1}$. 
   




\State \label{winbid} \emph{Aggregator} chooses the best point through a private auction: 

$x_{i^*} \leftarrow \textsc{ PrivAuction}(b_i: i\in [n])$
 \State \emph{Aggregator} verifies the data point against the bid value and updates $D_s \leftarrow D_s \cup x_{i^*}$.
\EndFor
\State {\bf return} Summary $D_s- D_{\mathrm{init}}$.
\Statex \hrulefill
\Procedure{PrivAuction}{$b_i: i\in [n]$} \label{proc:priv}
\State {\em Aggregator} orders the data owners, as $D'_1, D'_2, \ldots, D'_K$, by their decreasingly bid values. 
\State \label{step:indep} Independently with probability $\mathbb{P}[x_i] = e^{-\epsilon_{auc} (i-1)}$, {\em Aggregator} asks for the point $x_i$. 
\State If a certain data point $x$ was chosen $\tau$ times by a data source (Step~\ref{alg3:line3}), \emph{Aggregator} asks for it.
\State 
\emph{Aggregator} chooses a point $x_{*}$, with maximum bid value $b_{*}$, from the pool of all the points obtained so far and not yet included in the summary.
\State {\em Data owners} disconsider all the points sent to the {\em Aggregator} in the future iterations. 
\label{proc:priv:end}
\EndProcedure
\end{algorithmic}
\end{algorithm}

\begin{theorem}\label{thm:final-guarantees}
  Let $a \in (0,1)$ and $\tilde{\delta} \in (0,1/e)$ be any fixed constants. In Algorithm~\ref{algm:protocol}, for  
 $|D_v| \geq  \frac{44\sqrt{2}\sqrt{d}\log{d} \log^2 p}{\varepsilon_v}$, $\lvert D_{\mathrm{init}} \rvert \geq 121*8 d^2 \log^2 d \log (\frac{1}{\tilde{\delta}}) \log^5 p $, $d\geq \frac{16 (\log 2N) (\log p)^2}{a^2}$ , and setting $\eta \leq \frac{1}{d}$, $T=d^2$, $\varepsilon_v=\frac{\epsilon}{16T}$, $\varepsilon_{\ell, T}=\frac{\epsilon}{\sqrt{16 T \ell\log \left(\frac{1}{\tilde{\delta}} \right) \log p}}$,
we obtain the following guarantees: \\
{\bf (Differential Privacy)} Releases of the aggregator to any data owner $i$ is $(\epsilon,\tilde{\delta})$-differentially private over all the iterations/epochs with respect to the datasets $\cup_{j \neq i} D_i$. Similarly, we have $(\epsilon, \tilde{\delta})$-differentially privacy over all the iterations w.r.t.~validation set $D_v$.\\
{\bf(Approximation Guarantee)} Let {\sc opt} denote an optimal summary set and $D_s$ be the set of points obtained by Algorithm~\ref{algm:protocol}. We have $J(D_s) \geq (1-\frac{1}{e})J(\textsc{opt}) - \Delta$, where $\Delta < O(\frac{\log p \sqrt{\ln{d}}}{\sqrt{d}}) + a+ \frac{1}{\epsilon \log p} < 1$. Barring the $\Delta$ additive error the guarantees are close to the non-private greedy algorithm.\\
{\bf (Parsimoniousness Guarantee)} Algorithm~\ref{algm:protocol} is $O(\frac{\log \frac{1}{\delta}}{\epsilon}K^{\frac{1}{3}})$-parsimonious, i.e., in computing a summary of size $p$, it needs to access at most $O(pK^{\frac{1}{3}})$ data points. 
\end{theorem}
\textbf{Differences between {\bf\sc PrivAuction} and the Exponential Mechanism:} There \textit{may} be a superficial resemblance between Step 13 in the {\sc PrivAuction} procedure of Algorithm \ref{algm:protocol} and the exponential mechanism. Actually, our private auction is significantly different. First note that the probability of choosing the
best bid is $1$ which is not the case with the exponential mechanism. Secondly,
while the exponential mechanism selects one approximately "best" point, we flip a coin for every bid whose bias has an exponentially decreasing relationship to the position of the bid in sorting order. Then, we choose multiple of them (instead
of one) and a key proof point is to show that we can restrict the number of the points chosen overall. Finally, the bias probabilities do not even depend on the bid value (i.e., "score") while it would be the case for exponential mechanism.

{\bf Extension to a Less Trusted Curator.} 
In our parsimonious curator model, the final summary dataset needs to be revealed to the trusted aggregator in order to train diverse models \textit{downstream}. In Section~\ref{sec:addith1} of the Appendix, we show that Algorithm \ref{algm:protocol} can be adapted to share just $\h_1(\mathbf{x})$ hashes of data points. We show that this approach has some interesting privacy guarantees, specifically that the aggregator can only know the pairwise Euclidean distances between the points and nothing more. These hashes would be useful to train kernel based models such as \emph{Support Vector Machines}.

\section{Experimental Evaluation}
We make an important observation that is crucial to obtain good performance in practice. According to Theorem \ref{thm:expg}, in order to control the additive error in approximating the query  $\frac{w(D,i)}{q}$, Algorithm \ref{algm:h2} needs:
(a) $T$ (the number of iterations) in Algorithm \ref{algm:h2} to be larger than $d^2$ to match the distribution $P_{\mathrm{avg}}$ to the empirical distribution of coordinate $i$ in the current summary $D_s$,
(b) $D_{\mathrm{init}}$, the size of the initial seed summary also needs to be large enough because of this (refer Theorem \ref{thm:final-guarantees}). 
Over multiple epochs of Algorithm \ref{algm:protocol} (Step \ref{step:epoch}) , we make the following changes to deal with these issues. \emph{First Epoch ($\ell = 1$): }
In practice, we `seed' the protocol with a small initial seed set $D_{\mathrm{init}}$ to satisfy (b) and set $T=T_{\mathrm{init}}$ to be large enough ($d^{1.5}$) to satisfy (a).
\emph{Subsequent Epochs ($\ell > 1$): } Clearly, the summary $D_s$ grows and hence (b) is satisfied. We set $T=T_{\mathrm{sub}}$ to be a constant for subsequent iterations. This may seem to contradict the requirement (a). However, we observe that $\h_2(\cdot)$ operates on a summary that is only differing in one point from the previous iteration. 
Intuitively, a single point addition results in a small shift in the empirical distribution. Small incremental changes to the empirical distribution need to be matched incrementally. Thus, it is sufficient to have a significantly smaller number of iterations than that in Theorem~\ref{thm:mwem}. Therefore, $T_{\mathrm{sub}}$ is set to be small.
We set the parameters of our algorithm as follows:
the RBF kernel parameter
$\gamma =0.1$, dimension of Rahimi-Recht hash function $\h_1(.)$ as  $d=140$.  We use two different $T$ parameters for different epochs given by $T_{\mathrm{init}} (=T,\ell=1)= d^{1.5} = 1656$ and $T_{\mathrm{subs}}(=T,\ell > 1) =  5$.   
$\varepsilon_v = 0.01$ is the $\epsilon$ parameter for $\h_2(\cdot)$ for the validation set  and  $\varepsilon_{\ell, T} $ is set for $\h_2(\cdot)$ on summaries $D_s$ over epochs $\ell$ as $0.05$ for $\ell = 1$, $\frac{0.01}{\sqrt{pT_{\mathrm{subs}}}}$ for $\ell >1.$

\noindent {\em Differential Privacy:} An important observation here is that we do not need to preserve the privacy of the seed set, since it can be completely random. We now bound the differential privacy of our parameters with respect to both the consumer data and the summary data points.
{\em Consumer Dataset ($D_v$):}
We compute $\h_2(\h_1(D_v))$ only once i.e., in the first epoch. This involves $T_{\mathrm{init}} = 1656$ iterations in Algorithm~\ref{algm:h2}, with $\varepsilon_v = 0.01$. Applying Theorem~\ref{thm:composition1} (in Appendix), we see that the total differential privacy measure $\varepsilon = 1.4$ (setting $\tilde{\delta} = 0.01$).
{\em Summary Dataset ($D_s$):}
Over $p$ epochs of Algorithm~\ref{algm:protocol}, we have $5$ iterations each with differential privacy $\frac{0.01}{\sqrt{pT_{subs}}}$. Thus, again by applying Theorem~\ref{thm:composition1} (in Appendix), we obtain a total differential privacy of $0.043$ (with $\tilde{\delta} = 0.0001$).

{\bf Experiments on Real World Datasets:}
We now back our theoretical results with empirical experiments. We compare three algorithms:
a) {\em Non-Private Greedy}, where the aggregator broadcasts the (exact) average of the hashed summary set (i.e., $\frac{W(D_s,i)}{q}$) and hashed validation set (i.e., $\frac{W(D_v,i)}{m}$). This is equivalent to the approach of \cite{kim2016examples}. 
b)  {\em Private Greedy}, which is the Algorithm~\ref{algm:protocol} with parameters set as above.
c) {\em Uniform Sampling}, where we draw equal number $\frac{p}{K}$ of required samples from each data provider to construct a summary of size $p$.
We empirically show that {\em private greedy} closely matches the performance of {\em non-private greedy} even under the strong differential privacy constraints. For comparison, we show that our algorithm outperforms {\em uniform sampling}. The motivation for choosing the latter as a candidate comes from the typical manner of using stochastic gradient descent approaches such as Federated Learning [\cite{mcmahan2016communication}] that perform uniform sampling. 
We experiment with two real world datasets. We discuss one of them, which is based on an Allstate insurance dataset from a \cite{Kaggle} competition. We show similar results for the MNIST dataset, that contains image data for recognizing hand written digits, in the Appendix~\ref{app:expts}.

\begin{figure}[h!]
  \begin{subfigure}[b]{0.46\linewidth}
    \includegraphics[width=7cm]{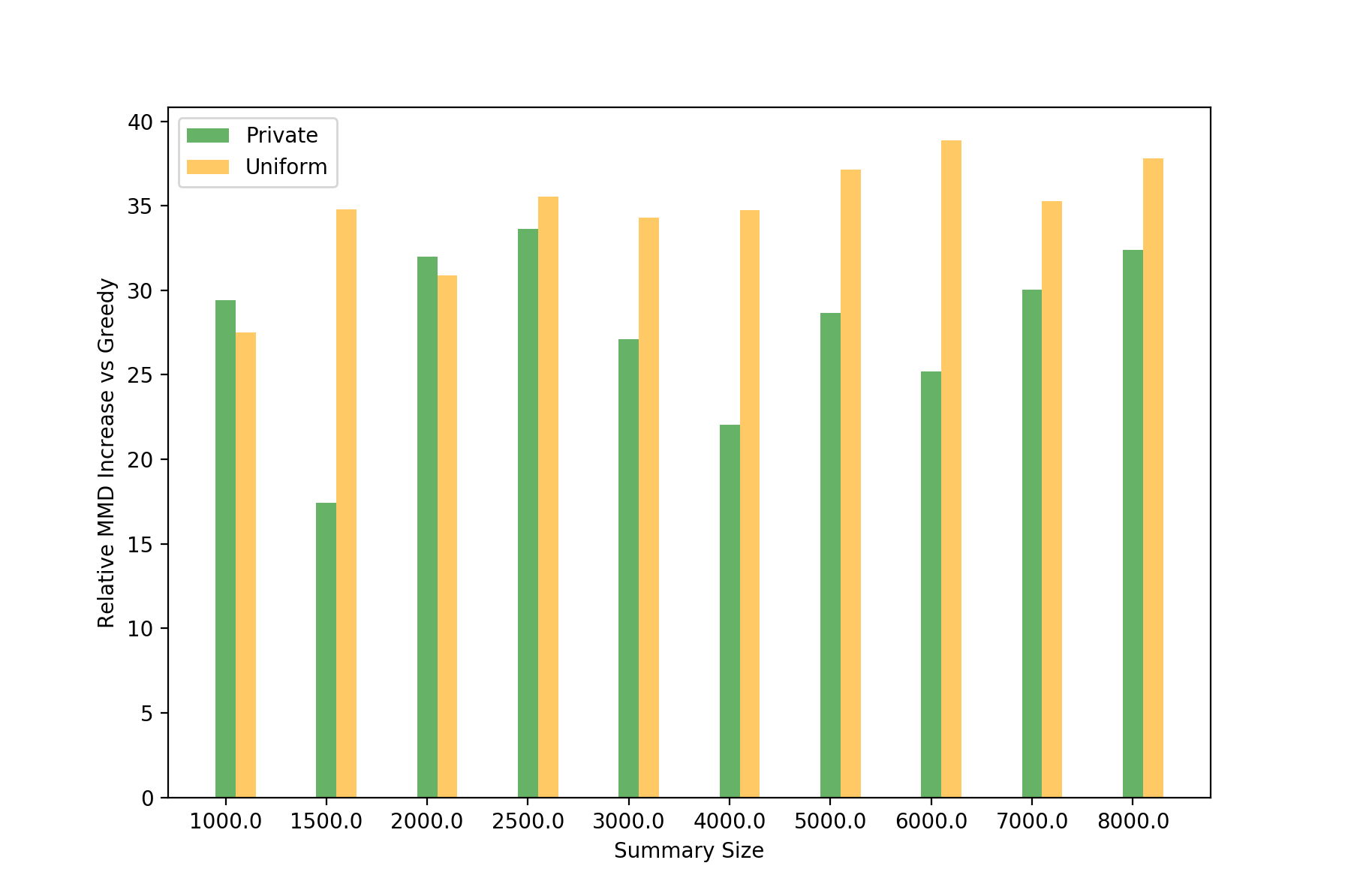}
  \end{subfigure}
  \begin{subfigure}[b]{0.46\linewidth}
    \includegraphics[width=7cm]{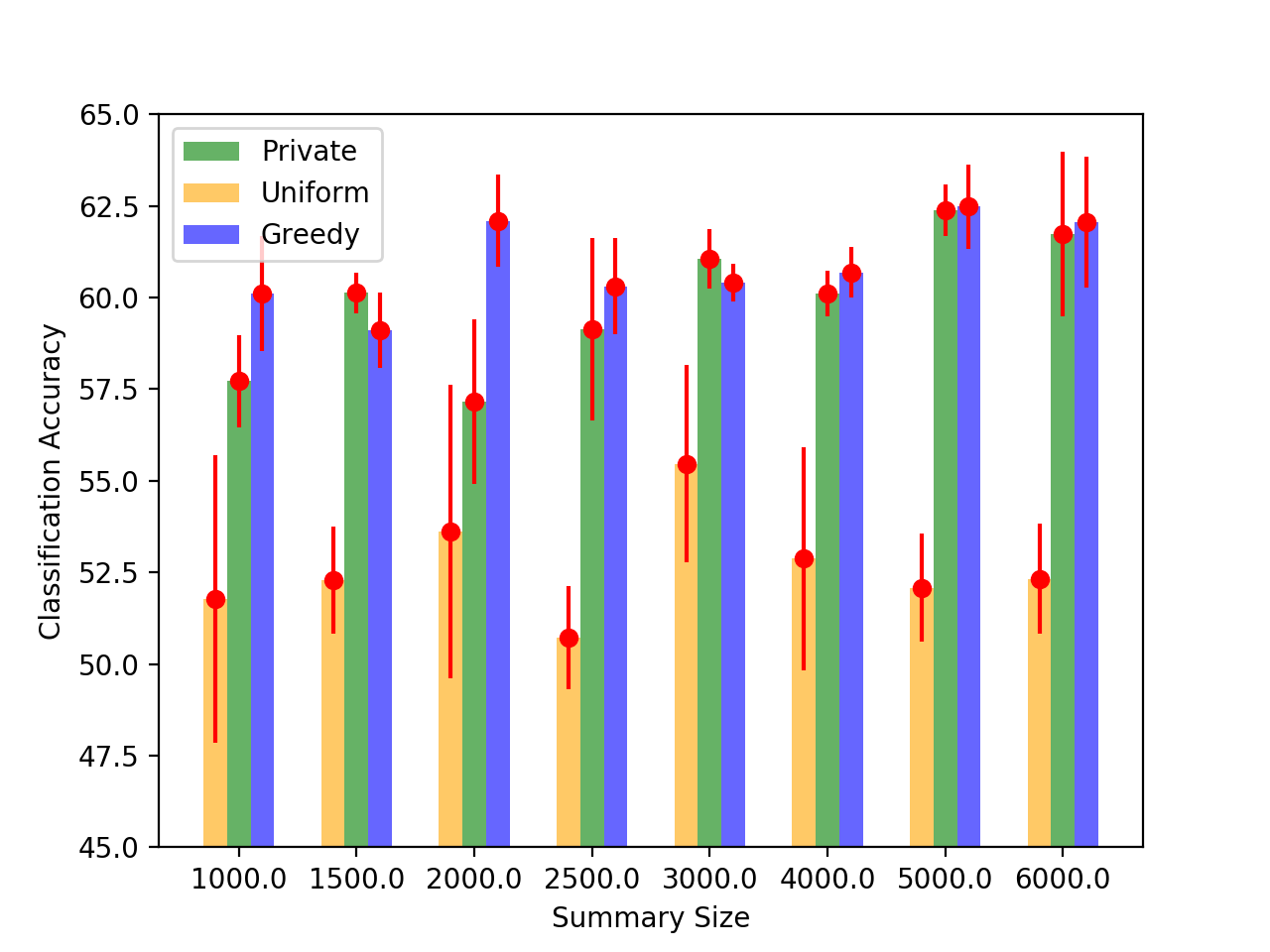}
  \end{subfigure}
  \caption{{\em All State Insurance Data:} (Top): Comparison of the percentage increase in $MMD^2$ of both the private and uniform sampling algorithms with respect to baseline greedy algorithm. Lower values indicate better performance. The private algorithm performs consistently better than uniform sampling. (Bottom): Comparison of the classification accuracy of the three algorithms using a  Linear SVM classifier. Higher numbers indicate better performance.  Our private algorithm outperforms uniform sampling by $\mathbf{6}$-$\mathbf{{10}\%}$ and closely matches the performance of the base line greedy algorithm.}
  \label{fig:allstate}
\end{figure}

\textbf {All State Insurance Data:}
The dataset contains insurance data of customers belonging to different states of the U.S. The objective is to predict labels of one of the all-state products. In our setup, we use data corresponding to two states - {\em Florida} and {\em Connecticut}. We have four data owner participants, and an aggregator. The data is split up as follows:
{\em Training data:}
The training data is comprised of all the Florida data and $70\%$ of the Connecticut data. The Florida data is split uniformly among the four data owners and Connecticut data is given to one of them. This allows us to create a skew in the data quality across different participants.
{\em Validation data:} From the remaining $30\%$ of Connecticut data we choose $25\%$ of data as the validation data set. Note that we remove the labels from this validation set before giving it to the consumer. 
{\em Testing data:} The remaining Connecticut data is set aside as testing data. Thus the training data is solely comprised of Connecticut data.
Further, we use around $150$ points of random seed data belonging to a different state ({\em Ohio}).
In our experiments, we vary the number of samples that need to be collected and compute the $MMD^2$ objective in each of these cases. In Figure~\ref{fig:allstate}, we compare the increase in $MMD^2$ with respect to greedy, i.e., $\frac{MMD^2(ALGM) - MMD^2(GREEDY)}{MMD^2(GREEDY)}\times 100$ where $ALGM$ is either our private greedy algorithm or the uniform sampling algorithm. Our results show that we consistently beat the uniform sampling algorithm while preserving differential privacy. In Figure~\ref{fig:allstate}, we compare the performance of these algorithms using a linear SVM. We find that the private algorithm while closely matching greedy beats uniform sampling  by $6\%$ to $10\%$.

\section{Discussion}
We consider a distributed data summarization problem in a transfer learning setting with privacy constraints. Different data owners have privacy constraints and a subset of points matching a target dataset needs to be formed. We provide a differentially private algorithm for this problem in the parsimonious curator setting, where the data owners do not wish to reveal information to other data owners and a curator entity can only access limited number of points.





\clearpage
\section*{Acknowledgement}
We thank Naoki Abe and Michele Franceshini for helpful discussions in the initial stages of this work. We also thank anonymous reviewers for their thoughtful suggestions that helped improve our presentation of the paper.

\bibliographystyle{icml2019}

\bibliography{ppa}

\begin{thebibliography}{35}
\providecommand{\natexlab}[1]{#1}
\providecommand{\url}[1]{\texttt{#1}}
\expandafter\ifx\csname urlstyle\endcsname\relax
  \providecommand{\doi}[1]{doi: #1}\else
  \providecommand{\doi}{doi: \begingroup \urlstyle{rm}\Url}\fi

\bibitem[Abadi et~al.(2016)Abadi, Chu, Goodfellow, McMahan, Mironov, Talwar,
  and Zhang]{abadi2016deep}
Abadi, M., Chu, A., Goodfellow, I., McMahan, H.~B., Mironov, I., Talwar, K.,
  and Zhang, L.
\newblock Deep learning with differential privacy.
\newblock In \emph{Proceedings of the 2016 ACM SIGSAC Conference on Computer
  and Communications Security}, pp.\  308--318. ACM, 2016.

\bibitem[Bassily et~al.(2014)Bassily, Smith, and Thakurta]{bassily2014private}
Bassily, R., Smith, A., and Thakurta, A.
\newblock Private empirical risk minimization: Efficient algorithms and tight
  error bounds.
\newblock In \emph{Foundations of Computer Science (FOCS), 2014 IEEE 55th
  Annual Symposium on}, pp.\  464--473. IEEE, 2014.

\bibitem[Chaudhuri et~al.(2011)Chaudhuri, Monteleoni, and
  Sarwate]{chaudhuri2011differentially}
Chaudhuri, K., Monteleoni, C., and Sarwate, A.~D.
\newblock Differentially private empirical risk minimization.
\newblock \emph{Journal of Machine Learning Research}, 12\penalty0
  (Mar):\penalty0 1069--1109, 2011.

\bibitem[Dwork et~al.(2006)Dwork, Kenthapadi, McSherry, Mironov, and
  Naor]{dwork2006our}
Dwork, C., Kenthapadi, K., McSherry, F., Mironov, I., and Naor, M.
\newblock Our data, ourselves: Privacy via distributed noise generation.
\newblock In \emph{Annual International Conference on the Theory and
  Applications of Cryptographic Techniques}, pp.\  486--503. Springer, 2006.

\bibitem[Dwork et~al.(2014{\natexlab{a}})Dwork, Nikolov, and
  Talwar]{dwork2014using}
Dwork, C., Nikolov, A., and Talwar, K.
\newblock Using convex relaxations for efficiently and privately releasing
  marginals.
\newblock In \emph{Proceedings of the thirtieth annual symposium on
  Computational geometry}, pp.\  261. ACM, 2014{\natexlab{a}}.

\bibitem[Dwork et~al.(2014{\natexlab{b}})Dwork, Roth,
  et~al.]{dwork2014algorithmic}
Dwork, C., Roth, A., et~al.
\newblock The algorithmic foundations of differential privacy.
\newblock \emph{Foundations and Trends{\textregistered} in Theoretical Computer
  Science}, 9\penalty0 (3--4):\penalty0 211--407, 2014{\natexlab{b}}.

\bibitem[Ganin \& Lempitsky(2014)Ganin and Lempitsky]{ganin2014unsupervised}
Ganin, Y. and Lempitsky, V.
\newblock Unsupervised domain adaptation by backpropagation.
\newblock \emph{arXiv preprint arXiv:1409.7495}, 2014.

\bibitem[Giraud \& Peschanski(2015)Giraud and Peschanski]{giraud2015dirac}
Giraud, B.~G. and Peschanski, R.
\newblock From" dirac combs" to fourier-positivity.
\newblock \emph{arXiv preprint arXiv:1509.02373}, 2015.

\bibitem[Gottlieb \& Khaled(2017)Gottlieb and Khaled]{mckinsey2}
Gottlieb, J. and Khaled, R.
\newblock Fueling growth through data monetization.
\newblock
  \url{https://www.mckinsey.com/business-functions/mckinsey-analytics/our-insights/fueling-growth-through-data-monetization},
  December 2017.

\bibitem[Gretton et~al.(2008)Gretton, Borgwardt, Rasch, Sch{\"{o}}lkopf, and
  Smola]{mmd}
Gretton, A., Borgwardt, K.~M., Rasch, M.~J., Sch{\"{o}}lkopf, B., and Smola,
  A.~J.
\newblock A kernel method for the two-sample problem.
\newblock \emph{CoRR}, abs/0805.2368, 2008.

\bibitem[Hamm et~al.(2016)Hamm, Cao, and Belkin]{hamm2016learning}
Hamm, J., Cao, Y., and Belkin, M.
\newblock Learning privately from multiparty data.
\newblock In \emph{International Conference on Machine Learning}, pp.\
  555--563, 2016.

\bibitem[Hardt \& Rothblum(2010)Hardt and Rothblum]{hardt2010multiplicative}
Hardt, M. and Rothblum, G.~N.
\newblock A multiplicative weights mechanism for privacy-preserving data
  analysis.
\newblock In \emph{Foundations of Computer Science (FOCS), 2010 51st Annual
  IEEE Symposium on}, pp.\  61--70. IEEE, 2010.

\bibitem[Hardt et~al.(2012)Hardt, Ligett, and McSherry]{hardt2012}
Hardt, M., Ligett, K., and McSherry, F.
\newblock A simple and practical algorithm for differentially private data
  release.
\newblock In \emph{Advances in Neural Information Processing Systems}, pp.\
  2339--2347, 2012.

\bibitem[Jukna(2011)]{jukna2011extremal}
Jukna, S.
\newblock \emph{Extremal combinatorics: with applications in computer science}.
\newblock Springer Science \& Business Media, 2011.

\bibitem[Kaggle(2014)]{Kaggle}
Kaggle.
\newblock Allstate purchase prediction challenge.
\newblock
  \url{https://www.kaggle.com/c/allstate-purchase-prediction-challenge}, 2014.

\bibitem[Kairouz et~al.(2017)Kairouz, Oh, and
  Viswanath]{kairouz2017composition}
Kairouz, P., Oh, S., and Viswanath, P.
\newblock The composition theorem for differential privacy.
\newblock \emph{IEEE Transactions on Information Theory}, 63\penalty0
  (6):\penalty0 4037--4049, 2017.

\bibitem[Kasiviswanathan et~al.(2011)Kasiviswanathan, Lee, Nissim,
  Raskhodnikova, and Smith]{kasiviswanathan2011can}
Kasiviswanathan, S.~P., Lee, H.~K., Nissim, K., Raskhodnikova, S., and Smith,
  A.
\newblock What can we learn privately?
\newblock \emph{SIAM Journal on Computing}, 40\penalty0 (3):\penalty0 793--826,
  2011.

\bibitem[Kifer et~al.(2012)Kifer, Smith, and Thakurta]{kifer2012private}
Kifer, D., Smith, A., and Thakurta, A.
\newblock Private convex empirical risk minimization and high-dimensional
  regression.
\newblock In \emph{Conference on Learning Theory}, pp.\  25--1, 2012.

\bibitem[Kim et~al.(2016)Kim, Khanna, and Koyejo]{kim2016examples}
Kim, B., Khanna, R., and Koyejo, O.~O.
\newblock Examples are not enough, learn to criticize! criticism for
  interpretability.
\newblock In \emph{Advances in Neural Information Processing Systems}, pp.\
  2280--2288, 2016.

\bibitem[Mardia \& Jupp(2009)Mardia and Jupp]{mardia2009directional}
Mardia, K.~V. and Jupp, P.~E.
\newblock \emph{Directional statistics}, volume 494.
\newblock John Wiley \& Sons, 2009.

\bibitem[McMahan et~al.(2016)McMahan, Moore, Ramage, Hampson,
  et~al.]{mcmahan2016communication}
McMahan, H.~B., Moore, E., Ramage, D., Hampson, S., et~al.
\newblock Communication-efficient learning of deep networks from decentralized
  data.
\newblock \emph{arXiv preprint arXiv:1602.05629}, 2016.

\bibitem[Mitrovic et~al.(2017)Mitrovic, Bun, Krause, and
  Karbasi]{mitrovic2017differentially}
Mitrovic, M., Bun, M., Krause, A., and Karbasi, A.
\newblock Differentially private submodular maximization: Data summarization in
  disguise.
\newblock In \emph{International Conference on Machine Learning}, pp.\
  2478--2487, 2017.

\bibitem[Nemhauser et~al.(1978)Nemhauser, Wolsey, and
  Fisher]{nemhauser1978analysis}
Nemhauser, G.~L., Wolsey, L.~A., and Fisher, M.~L.
\newblock An analysis of approximations for maximizing submodular set
  functions—i.
\newblock \emph{Mathematical Programming}, 14\penalty0 (1):\penalty0 265--294,
  1978.

\bibitem[Nissim \& Stemmer(2017)Nissim and Stemmer]{nissim2017clustering}
Nissim, K. and Stemmer, U.
\newblock Clustering algorithms for the centralized and local models.
\newblock \emph{arXiv preprint arXiv:1707.04766}, 2017.

\bibitem[Papernot et~al.(2016)Papernot, Abadi, Erlingsson, Goodfellow, and
  Talwar]{papernot2016semi}
Papernot, N., Abadi, M., Erlingsson, U., Goodfellow, I., and Talwar, K.
\newblock Semi-supervised knowledge transfer for deep learning from private
  training data.
\newblock \emph{arXiv preprint arXiv:1610.05755}, 2016.

\bibitem[Pathak et~al.(2010)Pathak, Rane, and Raj]{pathak2010multiparty}
Pathak, M., Rane, S., and Raj, B.
\newblock Multiparty differential privacy via aggregation of locally trained
  classifiers.
\newblock In \emph{Advances in Neural Information Processing Systems}, pp.\
  1876--1884, 2010.

\bibitem[Rahimi \& Recht(2008)Rahimi and Recht]{rahimi2008random}
Rahimi, A. and Recht, B.
\newblock Random features for large-scale kernel machines.
\newblock In \emph{Advances in neural information processing systems}, pp.\
  1177--1184, 2008.

\bibitem[Rubinstein et~al.(2012)Rubinstein, Bartlett, Huang, and
  Taft]{rubinstein2012learning}
Rubinstein, B.~I., Bartlett, P.~L., Huang, L., and Taft, N.
\newblock Learning in a large function space: Privacy-preserving mechanisms for
  svm learning.
\newblock \emph{Journal of Privacy and Confidentiality}, 4\penalty0
  (1):\penalty0 65--100, 2012.

\bibitem[Sarpatwar et~al.(2019)Sarpatwar, Ganapavarapu, Shanmugam, Rahman, and
  Vacul{\'{\i}}n]{sarpatwar2019blockchain}
Sarpatwar, K.~K., Ganapavarapu, V.~S., Shanmugam, K., Rahman, A., and
  Vacul{\'{\i}}n, R.
\newblock Blockchain enabled {AI} marketplace: The price you pay for trust.
\newblock In \emph{{IEEE} Conference on Computer Vision and Pattern Recognition
  Workshops, {CVPR} Workshops 2019, Long Beach, CA, USA, June 16-20, 2019},
  pp.\ ~0, 2019.

\bibitem[Shokri \& Shmatikov(2015)Shokri and Shmatikov]{shokri2015privacy}
Shokri, R. and Shmatikov, V.
\newblock Privacy-preserving deep learning.
\newblock In \emph{Proceedings of the 22nd ACM SIGSAC conference on computer
  and communications security}, pp.\  1310--1321. ACM, 2015.

\bibitem[Song et~al.(2013)Song, Chaudhuri, and Sarwate]{song2013stochastic}
Song, S., Chaudhuri, K., and Sarwate, A.~D.
\newblock Stochastic gradient descent with differentially private updates.
\newblock In \emph{Global Conference on Signal and Information Processing
  (GlobalSIP), 2013 IEEE}, pp.\  245--248. IEEE, 2013.

\bibitem[Talwar et~al.(2015)Talwar, Thakurta, and Zhang]{talwar2015nearly}
Talwar, K., Thakurta, A.~G., and Zhang, L.
\newblock Nearly optimal private lasso.
\newblock In \emph{Advances in Neural Information Processing Systems}, pp.\
  3025--3033, 2015.

\bibitem[Thakurta(2013)]{thakurta2013differentially}
Thakurta, A.~G.
\newblock \emph{Differentially private convex optimization for empirical risk
  minimization and high-dimensional regression}.
\newblock The Pennsylvania State University, 2013.

\bibitem[Tzeng et~al.(2014)Tzeng, Hoffman, Zhang, Saenko, and
  Darrell]{tzeng2014deep}
Tzeng, E., Hoffman, J., Zhang, N., Saenko, K., and Darrell, T.
\newblock Deep domain confusion: Maximizing for domain invariance.
\newblock \emph{arXiv preprint arXiv:1412.3474}, 2014.

\bibitem[Wu et~al.(2016)Wu, Kumar, Chaudhuri, Jha, and
  Naughton]{wu2016differentially}
Wu, X., Kumar, A., Chaudhuri, K., Jha, S., and Naughton, J.~F.
\newblock Differentially private stochastic gradient descent for in-rdbms
  analytics.
\newblock \emph{arXiv preprint arXiv:1606.04722}, 2016.

\end{thebibliography}

\clearpage

\appendix
\section{Missing Definitions}
\label{appendix:missing}

We provide the definition for the distribution version of maximum mean discrepancy.
\begin{definition}\label{def2}
Maximum mean discrepancy (MMD) between two distributions $P$ and $Q$ is defined as:
 \begin{align}
  \mathrm{MMD}^2(P,Q) & = \mathbb{E}_{X,X' \sim P}[k(X,X')] +\mathbb{E}_{Y,Y'\sim Q}[k(Y,Y')]  \nonumber \\
  \hfill & - 2 \mathbf{E}_{X \sim P,Y \sim Q}[k(X,Y)]
 \end{align}
where $k(\cdot,\cdot)$ is a kernel function underlying an RKHS (Reproducing Kernel Hilbert Space) function space such that $k(x,y)=k(y,x)$ and $k(\cdot,\cdot)$ is positive definite. 
\end{definition}

\section{Quantization}
\textbf{Quantization Function}\label{sec:quant}
Define a grid $S$ of points $S=\{-1, -1+\eta, -1+2\eta ...\ldots 1-\eta, 1\}$, where we assume $2/\eta$ is an integer for convenience. Define a random quantization function $Q:[-1,1] \rightarrow S$ as follows:
\begin{align}\label{eqn:quantize}
Q(x) = \left \lbrace
\begin{array}{cc}
-1+k \eta, & w.p. \frac{(k+1) \eta -1-x }{\eta} \\
-1+ (k+1)\eta & w.p. \frac{x + 1- k \epsilon}{\eta}
\end{array} \right\rbrace
\end{align}
where $k=\lfloor (x+1)/\eta \rfloor$. Here, the value $x$ is quantized to one of the two nearest points from $S$ with probabilities chosen carefully to make sure that the expected quantization error is $0$. Now, we consider the quantized data set $D_Q= \left\lbrace  Q\left( \sqrt{\frac{d}{2}} \mathbf{v}_i \right) \right \rbrace_{i=1}^{q}$. Observe that $D_Q \in S^{q \times d}$. Let $\tilde{v}_1, \tilde{v}_2 \ldots \tilde{v}_q \in S^{1 \times d}$ be the quantized vectors in $D_Q$. Let $w(D_Q,i)= \sum_{j=1}^q \tilde{v}_{ji}$.


\section{Approximation, Efficiency and Privacy Guarantees for the Protocol}

\textbf{Guarantees for $\h_2(\cdot)$:}
Now, we prove approximation and privacy guarantees for the hash function $\h_2(\cdot)$ with respect to the input dataset it operates on. We  observe that computing $\h_2(\cdot)$ involves maintaining a distribution over $\lvert \frac{2}{\eta} \rvert^d$ variables which is exponentially large. We first prove that we need only linear $O\left(\frac{d}{\eta} \right)$ memory and update time to maintain the different distributions.

\begin{lemma}\label{lem:productprop}
  In Algorithm \ref{algm:h2}, for all $0 \leq t \leq T$, $P_t(\mathbf{s})$ needs
$O\left(\frac{d}{\eta} \right)$ memory and update time.
\end{lemma}
\begin{proof}
It is enough to prove that distribution $P_t(\mathbf{s})$ satisfies the following two properties:

 a) (Product Distribution):
      $P_t(\mathbf{s})= \prod \limits_{i=1}^d P_t(s_i),~ \forall t $ here $P_t(s_i)$ is the marginal distribution on the coordinate $i$.
      
 b) (Marginal Update):
      $P_t(s_j)=P_{t-1}(s_j),~j \neq i(t)$.
      $P_t(s_j)=P_{t-1}(s_j)\exp \left[ s_j \left(\mu_j - w(P_{t-1},j)/2q \right) \right], j = i(t)$.

We first prove (a) by induction. The base case is true since the initial distribution is uniform. Now suppose it is true for some $t-1$, with $t>1$. 
 \begin{align}
 P_t(\mathbf{s}) &= \frac{P_{t-1}(\mathbf{s}) \exp(s_{i(t)}\frac{\mu_{i(t)} - w(P_{t-1}, i(t))}{2q})}{\sum_{\mathbf{s}}P_{t-1}(\mathbf{s})\exp(s_{i(t)} \frac{\mu_{i(t)} - w(P_{t-1}, i(t))}{2q})} \nonumber \\
 &=\left[\Pi_{i\neq i(t)}P_{t-1}(s_i)\right] * \nonumber \\
 \hfill & \left [\frac{P_{t-1}(s_{i(t)}) \exp(s_{i(t)} \frac{\mu_{i(t)} - w(P_{t-1}, i(t))}{2q})}{\sum_{\mathbf{s}}P_{t-1}(\mathbf{s})\exp(s_{i(t)} \frac{\mu_{i(t)} - w(P_{t-1}, i(t))}{2q})}\right]
 \end{align}
 
 Now, 
 \begin{align}
     &\sum_{\mathbf{s}}P_{t-1}(\mathbf{s})\exp(s_{i(t)}(\mu_{i(t)} - w(P_{t-1}, i(t)))/2q) \nonumber \\
     &= \sum_{s}\sum_{\mathbf{s}(i(t)) =s}P_{t-1}(\mathbf{s}|\mathbf{s}(i(t))=s)\exp(s \frac{\mu_{i(t)} - w(P_{t-1}, i(t))}{2q})\nonumber \\
     &=\sum_{s}\exp(s\frac{\mu_{i(t)} - w(P_{t-1}, i
     (t))}{2q})\sum_{\mathbf{s}_{i(t)} =s}P_{t-1}(\mathbf{s}|s_{i(t)}=s)\nonumber \\ 
     &=  \sum_{s}\exp(s(\mu_{i(t)} - w(P_{t-1}, i(t)))/2q) P_{t-1}(s_{i(t)}=s) 
 \end{align}
 It follows that the summation expression only depends on the coordinate $i(t)$ and hence we have decomposed $P_t(\mathbf{s})$ into distributions that are dependent only on the coordinates. Now (b) follows by computing the marginal distributions on each coordinate.
\end{proof}


Now, we prove an additive approximation guarantee for every coordinate of $\h_2(D,\varepsilon)$.
\begin{theorem}\label{thm:expg}
 (Expected Approximation Guarantee) Algorithm \ref{algm:h2} has the following approximation guarantee: 
\begin{align}
  &\mathbb{E}\left [\max_{i \in [d]}\left\lvert \frac{1}{q}w(D,i) -  \sqrt{\frac{2}{d}} \frac{1}{q} w(P_{\mathrm{avg}},i) \right \rvert \right] \leq  2 \sqrt{\frac{2 \log (2/\eta)}{d^2}} + 11\sqrt{2} \frac{\log d}{q\varepsilon \sqrt{d}} +  \frac{4}{d}+ 2 d \exp(-q/4) + \eta \nonumber
\end{align}

\end{theorem}

After the quantization step,  the algorithm for $\h_2(\cdot)$ (Algorithm \ref{algm:h2}) follows steps similar to the MWEM algorithm of \cite{hardt2012} but applied to the vectors in dataset $D_Q$. The different scalar queries on this data set are essentially the sums of the vectors in $D_Q$ along each of the $d$ coordinates. Therefore, we have the following theorem from \cite{hardt2012}, adapted to our case where the data set is $D_Q$ and the set of queries are the marginal sums $w(D_{Q},i)$. This gives the following guarantee:

\begin{theorem}\label{thm:mwem}
\cite{hardt2012} For any constant $c\geq 1$, with probability at least $1-\frac{2T}{d^c}$, Algorithm \ref{algm:h2} produces $P_{\mathrm{avg}}$ such that: $\max_{i \in [d]} \lvert w(D_Q,i) - w(P_{\mathrm{avg}},i)\rvert \leq 2q \sqrt{\frac{d \log \lvert S \rvert}{T}} + (3c+2) \frac{ \log d}{\varepsilon}.
$
\end{theorem}
\begin{proof}
This follows directly from \cite{hardt2012}, where we set the distribution support to be $|S|^d$ and support of every entry in $D_Q$ to be from $[-q,q]$. 
\end{proof}

Now, we provide an approximation guarantee for the quantization step using the $Q$ function.
\begin{lemma}\label{lem:chernoff}
  $\mathbb{E}[ w(D_Q,i)]=  \sqrt{\frac{d}{2}} w(D,i)$. With probability at least $1- 2d \exp(-\frac{q}{4})$, we have the following approximation: $\left \lvert \frac{1}{q}\sqrt{\frac{2}{d}}w(D_Q,i) - \frac{1}{q} w(D,i) \right \rvert \leq \eta  $ 
\end{lemma}
\begin{proof}
Every variable $\tilde{v}_{ji} - \sqrt{2}{d}v_{ji}$ is an independent mean zero random variable bounded in the interval $[-\eta,\eta]$. Therefore, applying Chernoff \cite{jukna2011extremal} bounds for bounded random variables with deviation $q \eta$ to the sum random variable $w(D_{Q},i)$ and combining it with a union bound on the $d$ coordinates yields the result.
\end{proof}
\begin{proof}[Proof of Theorem \ref{thm:expg}]
The theorem statement follows from the following: a) $\lvert \frac{1}{q}w(D,i) - \sqrt{\frac{2}{d}} \frac{1}{q} w(P_{\mathrm{avg}},i) \rvert \leq 2$ in the worst case and b) Lemma \ref{lem:chernoff} and choosing the parameters $T= d^2,c=3$ in Theorem \ref{thm:mwem} .
\end{proof}

\textbf{Final Differential Privacy and Approximation Guarantees:}
 We now describe the choices of different parameters in our protocol, including, $\epsilon_{\ell,T}$ over various epochs. In each of the $p$ epochs (note that $p$ is the final summary size), we apply Algorithm~\ref{algm:h2}.  
 In Theorem~\ref{thm:final-guarantees}, we  prove that releases of aggregator to any data owner $i$ in our protocol are $\epsilon$-differentially private (using the composition theorem from ~\cite{kairouz2017composition}) with respect to data sets of all other data owners except $i$. 
Further, we also bound the final expected additive error of our protocol over multiple rounds. Hence, using the following corollary (of a theorem due to Nemhauser, Wolsey and Fisher) we obtain approximation guarantees closely matching the greedy algorithm. 
\begin{theorem}(Corollary of \cite{nemhauser1978analysis})
\label{wolsey}Given a non-negative, monotone, submodular function $f:2^U\rightarrow \mathbb{R}^+\cup \{0\}$. Let $OPT$ be the optimal subset maximizing $f$ such that $|OPT| \leq p$. Similarly, let $A$ be the subset produced by greedy  algorithm such that the additive error in the marginal gain in iteration $i$ is $\Delta_i$. Then,
 $f(A) \geq (1-e^{-1})f(OPT) - \sum_{i\in [p]} \Delta_i$
\end{theorem}
  
\section{Proof of Theorem~\ref{thm:final-guarantees}}

We first prove the following differential privacy guarantees on various participant releases:
\begin{theorem}\label{thm:composition2}
For any fixed $\frac{1}{e}>\tilde{\delta}>0$, the releases of the aggregator during Algorithm~\ref{algm:protocol} to the any data owner $i$ is $(\varepsilon,\tilde{\delta})$-differentially private over all the iterations/epochs with respect to $\cup_{j \neq i} D_i$ when we set $\epsilon_{\ell,T} = \frac{\epsilon}{\sqrt{16 T \ell\log \left(\frac{1}{\tilde{\delta}} \right) \log p}}$. Similarly, we have $(\epsilon, \tilde{\delta})$-differentially privacy over all the iterations with respect to validation set $D_v$, when we set $\varepsilon_v = \frac{\epsilon}{\sqrt{16T}}$.  
\end{theorem}

We quote a recent result on composition theorems for differential privacy first.
\begin{theorem}{\cite{kairouz2017composition}}
\label{thm:composition1}
For any $\epsilon_\ell >0$, $\delta_\ell \in [0,1]$ for any $\ell\in \{1,2,\ldots, k\}$ and $\tilde{\delta}\in [0,1/e]$, the class $(\epsilon_\ell,\delta_\ell)$-differentially private mechanisms satisfy $(\tilde{\epsilon}_{\tilde{\delta}}, 1 - (1-\tilde{\delta})\Pi_{\ell=1}^{k} (1-\delta_\ell) )$-differential privacy under $k$-fold adaptive composition, for $\tilde{\epsilon}_{\tilde{\delta}} = $ 

\begin{align}
    \begin{split}
&\min\left\{ \sum_{\ell=1}^k \epsilon_\ell,\sum_{\ell=1}^k \frac{(e^{\epsilon_\ell} -1)\epsilon_\ell}{e^{\epsilon_\ell} +1)} +\sqrt{\sum_{\ell=1}^k 2\epsilon_{\ell}^2 \log \left(\frac{1}{\tilde{\delta}} \right)},
\right. \nonumber\\
&
\left .\sum_{\ell=1}^k \frac{(e^{\epsilon_\ell} -1)\epsilon_\ell}{(e^{\epsilon_\ell} +1)} +\sqrt{\sum_{\ell=1}^k 2\epsilon_{\ell}^2 \log \left( {e + \frac{\sqrt{\sum_{\ell=1}^k 2\epsilon_{\ell}^2}}{\tilde{\delta}}}\right)}
\right\}
    \end{split}
\end{align}
\end{theorem} 

\begin{proof}[Proof of Theorem \ref{thm:composition2}]

There are two types of releases by the aggregator to the data providers, over various iterations and we bound the differential privacy for these releases individually. 

\begin{enumerate}
    \item Releases of hashes $\h_1(\cdot)$ and $\h_2(\cdot)$ over multiple iterations.
    \item Release of information in the process of collecting data points from ``winner'' data sources. 
\end{enumerate}




Let us now analyze differential privacy of releases of type $2$. We set $\epsilon_{auc} = \frac{\epsilon}{3\sqrt{2}\log{\frac{1}{\delta}}}{K^{\frac{-1}{3}}}$ and $\tau = K^{\frac{2}{3}}$.  For the analysis of differential privacy from the perspective of data source $j$, consider two neighboring datasets $D=\cup_{j \neq i} D_i$ and $D'=D \cup \{x\}$.
Let us assume that $x$ belongs to data source $i' \neq j$. When $D'$ is involved,
define an iteration as {\em bad} if (a) $x$ is chosen by a data owner as marginally the best point in $D_{i'}$ (b) $x$ is not chosen by the aggregator.
By the virtue of our auction mechanism, there are at most $\tau$ such bad iterations, beyond which the point $x$ is chosen by the aggregator. 

The key point to note is that if an iteration is not bad, then the output distribution, i.e., the probabilities of chosen points by the aggregator, remains unchanged compared to the case when $D$ is involved. 

Further, in a bad iteration, the bid-value position of the data source $j$ can change by at most 1, say from $j$ to $j+1$ and thus the probability of choosing the data source $D_j$'s point can change by a factor of at most $\frac{e^{-(j-1)\epsilon_{auc}}}{e^{-j\epsilon_{auc}}} = e^{\epsilon_{auc}}$. Thus, the aggregator's queries for the private auction to data source $j$ for these iterations are $\epsilon_{auc}$-differentially private. 

Now, applying Theorem~\ref{thm:composition1}, we have:

\begin{align}
\sum_{b=1}^\tau  \frac{(e^{\epsilon_{auc}} -1)\epsilon_{auc}}{(e^{\epsilon_{auc}} +1)} & \leq 
\sum_{b=1}^\tau  \epsilon_{auc}^2  = \frac{\epsilon^2}{18 \log  \frac{1}{\tilde{\delta}}} \leq \epsilon^2/18
\nonumber \\
\end{align}
and 
\begin{align}
\sqrt{\sum_{b=1}^\tau 2\epsilon_{auc}^2 \log \left(\frac{1}{\tilde{\delta}} \right)} & = \sqrt{  \frac{2\epsilon^2\log \frac{1}{\tilde{\delta}} }{18\log^2 \left(\frac{1}{\tilde{\delta}} \right)}} \nonumber \\
 \hfill & \leq \frac{\epsilon}{3}
\end{align}

In Algorithm \ref{algm:h2}, steps $6$ and $7$ together release $i(t)$ and $\mu_i(t)$ (that are function of the final summary $D_s$) which is used in the computation of $P_t(\mathbf{s})$ which is used in the release from the aggregator to the data owners. Each of them is $\varepsilon$ differentially private. However, the $\ell$-th call to Algorithm \ref{algm:h2} by the protocol \ref{algm:protocol} uses $\varepsilon=\epsilon_{\ell,T}$. There are $T$ steps inside each call. 

We now set $\epsilon_{\ell,T} = \frac{\epsilon}{\sqrt{36 T \ell\log \left(\frac{1}{\tilde{\delta}} \right) \log p}}$, for iteration $\ell$
and  apply Theorem~\ref{thm:composition1} over all iterations. Firstly, note that by basic calculus, for $x\geq 0$, $\frac{e^x - 1}{e^x+1} \leq x$. This is because, setting $f(x) = (e^x - 1) - x(e^x + 1)$ has $f(0) = 0$ and $f'(x) = e^x - (e^x +1) - x (e^x+1) <= 0$. 

Thus, we have,
\begin{align}
\sum_{\ell=1}^p \sum_{t=1}^{2T} \frac{(e^{\epsilon_{\ell,T}} -1)\epsilon_{\ell,T}}{(e^{\epsilon_{\ell,T}} +1)} & \leq 
\sum_{\ell=1}^p \sum_{t=1}^{2T} \epsilon_{\ell,T}^2 \nonumber \\
\hfill &= \sum_{\ell=1}^p \frac{\epsilon^2}{18\log \left(\frac{1}{\tilde{\delta}} \right)\ell\log p} \nonumber \\
 & \leq
\frac{\epsilon^2}{18\log \left(\frac{1}{\tilde{\delta}} \right)\log p} \left(\sum_{\ell=1}^p \frac{1}{\ell}\right) \leq \frac{\epsilon^2}{18}
\end{align}
and 
\begin{align}
\sqrt{\sum_{\ell=1}^p \sum_{t=1}^{2T} 2\epsilon_{\ell,T}^2 \log \left(\frac{1}{\tilde{\delta}} \right)} & = \sqrt{ \sum_{\ell=1}^p  \frac{2\log \frac{1}{\tilde{\delta}} \epsilon^2}{18\log \left(\frac{1}{\tilde{\delta}} \right)\ell\log p}} \nonumber \\
 \hfill & \leq \frac{\epsilon}{3}
\end{align}

By Theorem~\ref{thm:composition1}, the protocol releases to any data owner is $(\tilde{\epsilon},\tilde{\delta})$-differentially private with respect to $D_s-D_i$ where 
$\tilde{\epsilon} \leq \frac{2\epsilon}{3} +   \frac{\epsilon^2}{9} $.
A similar computation shows that the releases of the aggregator during the protocol is $(\epsilon,\tilde{\delta})$-differential private with respect to the validation set $D_v$. 
\end{proof}
Now, we bound the overall expected additive error of our protocol. 
Define {\sf err}($E$) as the expected additive error in computing an expression $E$. 

\begin{lemma}\label{lem:error}
Suppose in the greedy algorithm, $S_q$ is the set of points chosen until iteration $q$ and $x_{q+1}$ be the new point chosen in iteration $q+1$. Let $D_v$ be the validation set. 
Let $\xi$ denote the maximum expected error in computing the terms, {\sf err}$(\frac{\sum_{i\in D_v}k(x_{q+1},y_i)}{m})$ $\leq \xi$ and {\sf err}$(\frac{\sum_{j\in S_q}k(x_{q+1},y_j)}{q})$ $\leq \xi$. Then the overall expected additive error of the algorithm is bounded by $\Delta \leq 7\xi\ln p $.
\end{lemma}
\begin{proof}
Consider the marginal increment in $J(.)$ in iteration $q+1$:
\begin{align}
&J(S_{q}\cup x_{q+1}) - J(S_q) \nonumber \\
&= \frac{2}{m}\left \{\frac{1}{q+1}\sum_{i\in D_v, j\in S_{q+1}}k(y_i,x_j) - \frac{1}{q}\sum_{i\in D_v, j\in S_{q}}k(y_i,x_j)\right \} \nonumber \\
&- \left \{\frac{1}{(q+1)^2} \sum_{i,j\in S_{q+1}} k(x_i,x_j)- \frac{1}{q^2} \sum_{i,j\in S_{q}} k(x_i,x_j)\right\} \nonumber \\
&= \frac{1}{q+1}\left\{\frac{2\sum_{i\in D_v}k(x_{q+1}, y_i)}{m} \right. \nonumber \\
& \left. - \frac{q}{q+1}\frac{(1+2\sum_{j\in S_q}k(x_{q+1},x_i))}{q}\right\}\nonumber \\ 
&+  \left(\frac{1}{q+1} - \frac{1}{q}\right)\frac{\sum_{q\in S_{q}}\sum_{i\in D_v}k(x_{q},y_i)}{m} + \nonumber \\
\hfill & \left ( \frac{1}{(q+1)^2} - \frac{1}{q^2}\right)\sum_{i,j\in S_q} k(x_i,x_j)
\end{align}
Now, we bound the additive error in computing this marginal increment as follows.
\begin{align}
&\textsf{err}(J(S_{q}\cup x_{q+1}) - J(S_q)) \nonumber \\
&\leq \frac{1}{q+1}\textsf{err}\left(\frac{2\sum_{i\in D_v}k(x_{q+1}, y_i)}{m}\right) + \nonumber \\ 
& \frac{q}{(q+1)^2}\textsf{err}\left(\frac{2\sum_{j\in S_q}k(x_{q+1},x_i)}{q}\right)\nonumber \\ 
&+  \left(\frac{1}{q} - \frac{1}{q+1}\right)\sum_{q\in S_{q}}\textsf{err}\left(\frac{\sum_{i\in D_v}k(x_{q},y_i)}{m}\right) +  \nonumber \\ 
& \left ( \frac{q}{q^2} - \frac{q}{(q+1)^2}\right)\sum_{q\in S_q}\textsf{err}\left (\frac{\sum_{j\in S_q} k(x_q,x_j)}{q}\right) \nonumber \\
&\leq \frac{2\xi}{q+1} + \frac{2q\xi}{(q+1)^2} + \frac{1}{q(q+1)}\sum_{q\in S_q} \xi +\frac{2q+1}{q(q+1)^2} \sum_{q\in S_q} \xi \nonumber \\
&= \frac{2\xi}{q+1} + \frac{2q\xi}{(q+1)^2} + \frac{1}{q(q+1)}q\xi +\frac{2q+1}{q(q+1)^2} q\xi\leq \frac{7\xi}{q+1}
\end{align}

By Theorem~\ref{wolsey}, the overall expected additive error in the greedy algorithm is bounded by $\Delta \leq \sum_{q\in[p]}\Delta_q \leq \sum_{q\in [p]}\frac{7\xi}{q+1}\leq 7\xi\ln p$
\end{proof}

\begin{lemma}\label{lem:boundingxi}
  Let $0<a < 1$ be a small fixed constant. Let $|D_v| \geq  \frac{11*4\sqrt{2}\sqrt{d}\log{d} \log^2 p}{\varepsilon_v}$, $\lvert D_{\mathrm{init}} \rvert \geq 121*8 d^2 \log^2 d \log (\frac{1}{\tilde{\delta}}) \log^5 p $, $d\geq \frac{16 (\log 2N) (\log p)^2}{a^2}$, $\eta \leq \frac{1}{d}$, we have $\Delta \leq 7\xi \ln p < O(\frac{\log p \sqrt{\ln{d}}}{\sqrt{d}}) + a+ \frac{1}{\epsilon \log p} < 1$. 
\end{lemma}
\begin{proof}
Let $N$ be total number of points in the system. First, we use a theorem from ~\cite{rahimi2008random}, to show that 
$\mathbb{P}(\sup \limits_{x_i,x_j} |\h_1(x_i)\cdot \h_1(x_j) - k(x_i,x_j)|\geq \varepsilon_{rr}) \leq \frac{1}{N^2}$. 
Indeed, for a {\em fixed} pair of points, $x_i,x_j$, it holds that:
$\mathbb{P}(|\h_1(x_i)\cdot \h_1(x_j) - k(x_i,x_j)|\geq \varepsilon_{rr}) \leq \exp{(-\frac{d\varepsilon_{rr}^2}{4})}$.
Thus, by union bound, and setting $d \geq \frac{16\log{2N}}{\varepsilon_{rr}^2}$, we have the above claim.

Now, from Theorem~\ref{thm:expg}, for iteration $\ell$ in the protocol,
we have the following guarantee:
\begin{align}
\textsf{err}\left(\frac{w(\h_1(D_s),i)}{q}\right) &\leq 2 \sqrt{\frac{2 \log (2/\eta)}{d^2}} + 11\sqrt{2} \frac{\log d}{q \epsilon_{\ell,T} \sqrt{d}} +   \nonumber \\
& \frac{4}{d}+ 2 d \exp(-q/4) + \eta 
\end{align}
Observe that at iteration $\ell$, $q= \ell+\lvert D_{\mathrm{init}}\rvert$ since this is the effective size of the summary. By the inequality between the arithmetic and geometric mean, we have: $ q \geq \sqrt{4 \ell \lvert D_{\mathrm{init}} \rvert}$. Now, we let $\lvert D_{\mathrm{init}} \rvert \geq 121*8 d^2 \log^2 d \log (\frac{1}{\tilde{\delta}}) \log^5 p$. Now, we set $\eta \leq \frac{1}{d}$. Then,
\begin{align}
   \textsf{err}\left(\frac{w(\h_1(D_s),i)}{q}\right) &\leq 2 \sqrt{\frac{2 \log (2/\eta)}{d^2}} + \frac{6}{d} + \frac{1}{\sqrt{d}\epsilon \log^2 p} \nonumber \\
   & = \Delta_{\mathrm{max}}\nonumber
\end{align}

Let $\h_1(\mathbf{x}_{q+1})[i]$ be the $i$-th coordinate of $\h_1(\mathbf{x}_{q+1})$. Observe that $\lvert \h_1(\mathbf{x}_{q+1})[i] \rvert \leq \sqrt{\frac{2}{d}}$.

We have the following expected additive error:
\begin{align}
& \textsf{err}\left(\frac{\sum_{j\in D_s}k(x_{q+1}, x_j)}{q}\right) \leq \varepsilon_{rr} + \nonumber \\
&  \textsf{err}\left(\frac{\sum_{j\in D_s}\h_1(x_{q+1})\cdot \h_1(x_j)}{q}\right) \leq \varepsilon_{rr} + \Delta_{\mathrm{max}} \sqrt{2d}.
\end{align}

We set $\epsilon_{err}=a/\log p$ for some small constant $a>0$. Therefore, $d \geq 16 \log 2N (\log p)^2/a^2$. Now, we have:
\begin{align}
\textsf{err}\left(\frac{\sum_{j\in D_s}k(x_{q+1}, x_j)}{q}\right) &\leq \varepsilon_{rr} + \Delta_{\mathrm{max}}\sqrt{2d} \nonumber \\
\hfill &= \frac{a}{\log p} +4 \sqrt{\frac{ \log (2/\eta)}{d}} + \nonumber \\
& \frac{6\sqrt{2}}{\sqrt{d}}+ \frac{1}{\epsilon \log^2 p} \nonumber \\ 
\hfill &=O(\frac{\sqrt{\ln{d}}}{\sqrt{d}}) + \frac{a}{\log p}+ \frac{1}{\epsilon \log^2 p}
\end{align}

Similarly, we can show that for validation set $D_v$, we need 
$|D_v| \geq  11*4\sqrt{2}\sqrt{d}\log{d} \log^2 p$. Now, since $\varepsilon_v = \frac{\varepsilon}{\sqrt{16d^2}}$, the $\Delta_{\mathrm{max}}$ bound holds for the validation term too.
\end{proof}

\begin{lemma}
\label{lem:auction}
In Algorithm~\ref{algm:protocol} the expected number of points accessed by the aggregator is $p(\frac{K}{\tau}+ \frac{1}{\epsilon_{auc}}) = (1+\frac{3\sqrt{2}\log{\frac{1}{\delta}}}{\epsilon})pK^{\frac{1}{3}} = O(\frac{p\log{\frac{1}{\delta}}}{\epsilon}K^{\frac{1}{3}})$
\end{lemma}
  \begin{proof}
  In the Step~\ref{step:indep} of the mechanism, the expected number of points chosen in each iteration is $\sum_{i \in [K]}\mathbb{P}(x_i) = \sum_{i \in [K]} e^{(i-1)\epsilon_{auc}} = (1-e^{-K\epsilon_{auc}})/(1-e^{-\epsilon_{auc}})\leq \frac{1}{1-e^{-\epsilon_{auc}}} \leq \frac{1}{\epsilon_{auc}}$. Thus in $p$ iterations, the expected number of points chosen $= \frac{p}{\epsilon_{auc}}$. In the second step, the maximum number of points that are the best for a data source more than $\tau$ times is $\frac{pK}{\tau}$. Thus the lemma follows.
  \end{proof}

\begin{proof}[Proof of Theorem \ref{thm:final-guarantees}]
The proof follows from the results in this section.

\end{proof}





\section{ Proof of Theorem~\ref{thm:submod}}
\label{app:submod}

We begin by quoting a known result from \cite{kim2016examples}.

\begin{theorem}{\cite{kim2016examples}}
\label{linear-form}
Let $\mathbf{H} \in \mathbb{R}^{g\times g}$ be element-wise non-negative and bounded with $h^* = \max{i,j\in [g]} h_{i,j} >0$. Define a binary matrix $\mathbf{E}$ with entries $e_{i,j} = 1$ if $h_{i,j} = h^*$ and $0$ otherwise. Similarly define $\mathbf{E}' = \mathbf{1} - \mathbf{E}$. Given the ground set $\mathcal{S}\subseteq 2^{[g]}$ consider the linear form:
$F(\mathbf{H},S) = \langle A(S), \mathbf{H}\rangle$ $\forall S\in \mathcal{S}$. Given $s = |S|$, define the functions:

$\alpha(g,s) = \frac{a(S\cup \{u\}) - a(S) }{b(S)},\quad \beta(g,s) = \frac{a(S\cup \{u\}) + a(S\cup \{v\}) - a(S\cup \{u,v\}) -a(S)}{b(S)+b(S\cup \{u,v\})}$, where
$a(S) = F(\mathbf{E},S)$ and $b(S) = F(\mathbf{E}',S)$ for all $u,v\in S$. Let $s^* = \max_{S\in \mathcal{S}}|S|$, we have

\begin{enumerate}
    \item $F(\mathbf{H},S)$ is monotone, if $h_{i,j} \leq h^* \alpha(g,s)$, $\forall 0\leq s \leq s^*$ 
\item $F(\mathbf{H},S)$ is submodular, if $h_{i,j} \leq h^* \beta(g,s)$, $\forall 0\leq s \leq s^*$ 
\end{enumerate}

\end{theorem}

\begin{proof}[Proof of Theorem \ref{thm:submod}]
Firstly, we show that the function $J(S)$ can be written in a linear form. Note that the same linear form used by~\cite{kim2016examples} would not work for our case.

We define $\mathbf{U}$ as the kernel matrix of all the points in $D_1\cup D_2\ldots D_k\cup D_v$. 

Now, we observe that our $J(S) = \langle \mathbf{A}(S),\mathbf{U}\rangle$, where $\mathbf{A}(S) = \frac{2}{m|S|} \mathbf{1}_{[i\in S]}\mathbf{1}_{[j\in V]} - \frac{1}{|S|^2}\mathbf{1}_{[i\in S]}\mathbf{1}_{[j\in S]}$. Let $\mathbf{E}$ as the binary matrix defined in Theorem~\ref{linear-form} with $\mathbf{H} = \mathbf{U}$.

We now compute $a(S) = \langle \mathbf{A}(S), \mathbf{E}\rangle$ and $b(S)=\langle \mathbf{A}(S), \mathbf{1} - \mathbf{E}\rangle$ values.

{\bf Computing $a(S)$:} 
\begin{align}
    a(S) = \langle \mathbf{A}(S), \mathbf{I}\rangle = \frac{2}{m|S|} 0 - \frac{1}{|S|^2} |S| = -\frac{1}{|S|}
\end{align}

{\bf Computing $b(S)$:}
\begin{align}
    b(S) &= \langle \mathbf{A}(S), \mathbf{1} - \mathbf{I}\rangle = \langle \mathbf{A}(S), \mathbf{1}\rangle - \langle \mathbf{A}(S), \mathbf{I}\rangle \nonumber \\
    \hfill &= \frac{2}{m|S|}|S|m - \frac{1}{|S|^2}|S|^2 + \frac{1}{|S|} = \frac{1}{|S|} + 1
\end{align}

Now, we show that the bounds on $\alpha(g,s)$ and $\beta(g,s)$ hold:
\begin{align}
 \alpha(g,s) = \frac{a(S\cup \{u\}) - a(S) }{b(S)} = \frac{\frac{1}{|S|} - \frac{1}{|S|+1}}{\frac{1}{|S|}+1} = \frac{1}{(1+|S|)^2}  
\end{align}

Further,
\begin{align}
\beta(g,s) &= \frac{a(S\cup \{u\}) + a(S\cup \{v\}) - a(S\cup \{u,v\}) -a(S)}{b(S)+b(S\cup \{u,v\})} \nonumber \\
&= \frac{-\frac{2}{|S|+1} +\frac{1}{|S|+2} + \frac{1}{|S|}}{\frac{1}{|S|}+1 + \frac{1}{|S|+2}+1} =\frac{1}{n^3+3n^2+n}
\end{align}

Thus, we have $k_{i,j} \leq \beta(g,s)k^*$ and hence the conditions of the Theorem~\ref{linear-form} are satisfied. Therefore, $J(S)$ is a monotone and submodular function.
\end{proof}


 \section{Additional Privacy Properties of $\h_1(\cdot):$}
 \label{sec:addith1}
Consider any data set $D_r$. Over the course of $p$ epochs, suppose the data source $r$ contributes $p_r$ points by winning bids at Line \ref{winbid} of Algorithm \ref{algm:protocol}. Let the points be $\mathbf{x}_1, \mathbf{x}_2, \ldots, \mathbf{x}_{p_r}$. We show that the joint probability density function of the random variables $\h_1(\mathbf{x}_1), \ldots, \h_1(\mathbf{x}_{p_r})$ depends only on the pairwise distances between the points, i.e. $  \lVert \mathbf{x}_u - \mathbf{x}_v \rVert ~\forall  u,v \in [1:p_r]$. 
In a strong information theoretic sense, this implies that the only information that can be gained about these points by the aggregator are the pairwise distances between the data points. The intention of usage of the hashes is to compute $k(\mathbf{x}_u,\mathbf{x}_v)= k(\lVert  \mathbf{x}_u - \mathbf{x}_v\rVert_2)$ approximately at the aggregator. Hence, the aggregator gains strictly no more information than it  needs. Consider the matrix $\left[\h_1(\mathbf{x}_1)^T, \ldots, \h_1(\mathbf{x}_{p_r})^T \right]^T$. Each column of this matrix is an i.i.d sample drawn from the distribution on the variables:  $\left[ \sqrt{2/d} \cos(\mathbf{w}^T\mathbf{x}_1+b) \ldots \sqrt{2/d} \cos(\mathbf{w}^T\mathbf{x}_{p_r}+b)\right]$ where $\mathbf{w} \sim N(0,2 \gamma\mathbf{I}_n ),~b \sim \mathrm{Uniform}[0,2\pi]$. In fact, we will analyze the joint characteristic function of the angles in a single column given by: $\left[ (\mathbf{w}^T\mathbf{x}_1+b) \mod 2\pi  \ldots (\mathbf{w}^T\mathbf{x}_{p_r}+b) \mod 2\pi \right]^T$. In an intuitive sense, these variables represent a randomly shifted  jointly Gaussian variables `wrapped' around a unit circle (usually called the wrapped distribution \cite{mardia2009directional}). The next theorem shows that the characteristic function 
depends only on the pairwise distance of the data points. 
\begin{theorem}\label{thm:char}
Let $\phi_w(\cdot)$ be the characteristic function of the wrapped distribution of the variables $\left[(\mathbf{w}^T\mathbf{x}_1+b) \mod 2\pi,  \ldots, (\mathbf{w}^T\mathbf{x}_{p_r}+b) \mod 2\pi \right]$. Then, we have:
a) $\forall \mathbf{s} \in \mathbb{R}^{p_r} -\mathbb{Z}^{p_r},$
$\phi_w(\mathbf{s})=0$. b) $\forall \mathbf{k} \in \mathbb{Z}^{p_r},~ \mathbf{1}^\mathbf{k} \neq 0,~\phi_w(\mathbf{k})=0$.
c) $\forall \mathbf{k} \in \mathbb{Z}^{p_r},~ \mathbf{1}^\mathbf{k} \neq 0$, $ \phi_w(\mathbf{k})=\phi(\mathbf{k})= \prod \limits_{i,j \in [1:p_r]} \left(\phi^{(i,j)}\right)^{m_{i,j}} $
where $m_{i,j}$ are some integers that depend on the vector $\mathbf{k}$ alone. Here, $\phi^{(i,j)}= \exp(-2 \gamma \lVert \mathbf{x}_i -\mathbf{x}_j \rVert_2^2) = k(\lVert \mathbf{x}_i -\mathbf{x}_j  \rVert_2)$.
\end{theorem}
\textbf{Remark:} We are not aware of any analysis of the joint distribution of multiple data point releases using Rahimi-Recht random features method for the RBF kernel. 
We use Fourier analysis, properties of multi-dimensional Dirac-combs \cite{giraud2015dirac} to prove the above theorem.

\subsection{Proof of Theorem \ref{thm:char}}

We first review results relating characteristic function of unwrapped distributions and the wrapped distributions. This relationships is due to some facts known about multi-dimensional Dirac Comb in standard Fourier Analysis. Let $p(\mathbf{v})$ be a density function defined on $\mathbb{R}^s$. Here $p(\cdot)$ is the unwrapped joint density function of the variables, $\left[ \mathbf{w}^T\mathbf{x}_1+b   \ldots \mathbf{w}^T\mathbf{x}_{p_r}+b \right]$.  Here, $\mathbf{v} \in \mathbb{R}^s$. The wrapped distribution of this density function is given by: $p_{w}(\mathbf{v})= \sum \limits_{\mathbf{k} \in \mathbb{Z}^n} p(\mathbf{v}+ 2 \pi \mathbf{k})$. Define the Dirac comb as:
 $\Delta_{2\pi}(\mathbf{v})= \sum \limits_{\mathbf{k} \in \mathbb{Z}^s} \delta(\mathbf{v}- 2 \pi \mathbf{k})$ where $\delta(\mathbf{v})= \prod_i \delta(v_i)$ and $\delta(\cdot)$ is a single dimensional Dirac-delta function. Although Dirac-delta functions are not rigorous as a real function, as a measure on the space $\mathbb{R}^s$, they are very well defined and rigorous.

It is known that the Fourier Series of the Dirac comb is given by:
  \begin{align}
     \Delta_{2\pi}(\mathbf{v}) = \frac{1}{(2\pi)^s} \sum \limits_{\mathbf{k} \in \mathbb{Z}^s} \exp(-i \mathbf{k}^T\mathbf{v}) 
  \end{align}
 
 Therefore, any wrapped distribution can be written in the following way:
 \begin{align}\label{Fourier}
     p_{w}(\mathbf{v}) &= \int p(\mathbf{v}')\Delta_{2 \pi} (\mathbf{v} - \mathbf{v}') d\mathbf{v}' \nonumber \\
    \hfill  &= \frac{1}{(2\pi)^s}\int p(\mathbf{v}') \sum \limits_{\mathbf{k} \in \mathbb{Z}^s} \exp(-i \mathbf{k}^T(\mathbf{v}-\mathbf{v}')) d\mathbf{v}' \\
     \hfill & = \frac{1}{(2 \pi)^s} \sum \limits_{\mathbf{k} \in \mathbb{Z}^s} \phi(\mathbf{k})  \exp(-i \mathbf{k}^T \mathbf{v})
 \end{align}
 
 Here, $\phi(\mathbf{k}) = \mathbb{E}_p[\exp(i \mathbf{k}^\mathbf{v})]$ is the characteristic function of the distribution $p(\cdot)$ on the integer lattice. Therefore, any wrapped distribution can be written as a Fourier series with Fourier Coefficients being the characteristic function evaluated at the integer lattice.
 
 Let $\phi_w(\cdot)$ be the characteristic function of the wrapped distribution. Further, $\phi_w(\mathbf{k})=\phi(\mathbf{k}),~ \forall \mathbf{k}\in \mathbb{Z}^{+}$ while $\phi_w(\mathbf{s})=0$ when $\mathbf{s} \in \mathbb{R}^s \ \mathbb{Z}^s$ is not on the integer lattice. This is very clear from the Fourier serier representation of the wrapped distribution as in (\ref{Fourier}).

 \begin{lemma}\label{lemzero}
   $\phi_w(\mathbf{k})=\phi(\mathbf{k})=0, ~ \forall \mathbf{k}: \mathbf{1}^T \mathbf{k} \neq 0$ when $p(\cdot)$ is the unwrapped joint distribution of $\left[\mathbf{w}^T\mathbf{x}_1+b \ldots \mathbf{w}^T{\mathbf{x}_{p_i}}+b \right]$ where $\mathbf{w} \sim {\cal N}(0,2\gamma \mathbf{I}_n)$ and $b \sim \mathrm{Uniform}[0,2\pi]$.
 \end{lemma}  
 \begin{proof}
   Let $\mathbf{X}= \left[\mathbf{x}_1^T \ldots \mathbf{x}_{p_r}^T \right]^T$. Therefore, variables $\mathbf{w}^T \mathbf{x}_j$ are jointly Gaussian with the covariance matrix $\Sigma=\mathbf{X}\mathbf{X}^T$. Given a fixed $b$, the conditional characteristic function over the integer lattice is given by: 
   \begin{align}\label{condchar}
    \phi_{\lvert b}(\mathbf{k})= \exp(i(\mathbf{k}^T \mathbf{1}) b) \exp(-\frac{1}{2} \mathbf{k}^T \Sigma \mathbf{k})
    \end{align}
This is the characteristic function of the standard multidimensional normal distribution.    
  
  $E_{b}[\exp(imb)]=0$ for an integer $m$ and $b \sim \mathrm{Uniform}[0,2\pi]$. Therefore, by (\ref{condchar}), we have the desired result.  
 \end{proof}
 
 We will show that the $\phi(\mathbf{k})$ is a function only of the pairwise distances between the points whenever $\mathbf{1}^T\mathbf{k}=0$. 
 
 \begin{lemma}\label{lemnonzero}
  Let $\mathbf{k}^{(i,j)}= [0 \ldots \underset{\mathrm{position~}i}{1} \ldots \ldots \underset{\mathrm{position~}j}{-1} ])$. Then, $\phi \left( \mathbf{k}^{(i,j)} \right)= \exp(-2 \gamma \lVert \mathbf{x}_i -\mathbf{x}_j \rVert_2^2)$. Further, whenever $\mathbf{k}^T\mathbf{1}=0$, $ \phi_w(\mathbf{k})=\phi(\mathbf{k})= \prod \limits_{i,j \in [1:p_r]} \left( \phi \left( \mathbf{k}^{(i,j)} \right) \right)^{m_{i,j}} $
where $m_{i,j}$ are some integers that depend on the vector $\mathbf{k}$ alone.
 \end{lemma}
 \begin{proof}[Proof of Lemma \ref{lemnonzero}]
   Whenever $\mathbf{k}^T\mathbf{1}=0$, by (\ref{condchar}) $\phi(\mathbf{k})$ is a function of $\lVert\mathbf{k}^T\mathbf{X} \rVert^2$. Let $\sum |k_i|=2t,~t \in \mathbb{Z}^+$. The sum of absolute values is an even integer because $\sum k_i=0$. Now, we can write $\lVert\mathbf{k}^T\mathbf{X} \rVert^2$ as follows:
 \begin{align}
  \lVert \mathbf{k}^T\mathbf{X} \rVert^2=\lVert \sum \limits_{j=1}^t (\mathbf{g}_i-\mathbf{h}_i) \rVert_2^2       
 \end{align}
 where $\mathbf{g}_i = \mathbf{x}_j$ for some $j \in [1:p_r]$ and $\mathbf{h}_i=\mathbf{x}_k$ for some $k \in [1:p_r]$. Because any distinct data point $\mathbf{x}_j$ is multiplied only by either positive or negative integers, clearly $ \{\mathbf{g}_i\}_{i=1}^t \bigcap \{\mathbf{h}_i\}_{i=}^t = \emptyset$.
 
 Now, we have:
 \begin{align}\label{terms1}
 \lVert \sum \limits_{j=1}^t (\mathbf{g}_i-\mathbf{h}_i) \rVert_2^2 &= \sum \limits_{j=1}^t \lVert (\mathbf{g}_j-\mathbf{h}_j) \rVert_2^2 + \nonumber \\
 \hfill & 2*\sum \limits_{j,j'} (\mathbf{g}_j-\mathbf{h}_j)^T (\mathbf{g}_{j'}-\mathbf{h}_{j'})
 \end{align}
 
 The first terms set of terms clearly are function of pairwise distances between points. Now we rewrite the cross terms as linear combination of pairwise distances in the following way.
  \begin{align} \label{terms2}
     2*(\mathbf{g}_j-\mathbf{h}_j)^T (\mathbf{g}_{j'}-\mathbf{h}_{j'}) 
     &= \lVert \mathbf{g}_j- \mathbf{h}_{j'} \rVert_2^2+ \lVert \mathbf{g}_{j'}- \mathbf{h}_{j} \rVert_2^2 \nonumber \\
     \hfill & - \lVert \mathbf{g}_j- \mathbf{g}_{j'} \rVert_2^2- \lVert \mathbf{h}_j- \mathbf{h}_{j'} \rVert_2^2
  \end{align}
  Hence, characteristic function can be written as pairwise distances between the data points.
  
  Let $\mathbf{k}^{(i,j)}= [0 \ldots \underset{\mathrm{position~}i}{1} \ldots \ldots \underset{\mathrm{position~}j}{-1} ])$. Then, $\phi \left( \mathbf{k}^{(i,j)} \right)= \exp(-2 \gamma \lVert \mathbf{x}_i -\mathbf{x}_j \rVert_2^2)$. These are exactly the kernel values that the Aggregator is interested in. By (\ref{terms1}) and (\ref{terms2}), it is clear that the characteristic function can be written in terms of powers of $\phi \left( \mathbf{k}^{(i,j)} \right)$, i.e.
    \begin{align}
       \phi(\mathbf{k})= \prod \limits_{i,j \in [1:p_r]} \left( \phi \left( \mathbf{k}^{(i,j)} \right) \right)^{m_{i,j}}   
    \end{align}
where $m_{i,j}$ are some integers that depend on the vector $\mathbf{k}$ alone.    
 \end{proof}
\begin{proof}[Proof of Theorem \ref{thm:char}] 
 The results of the two lemmas above prove the theorem. 
\end{proof}
 
\section{Additional Experiments}
\label{app:expts}
As discussed before, we set the parameters of our algorithm as in Table~\ref{table:parameters}.

\begin{table}[H]
    \centering
\begin{tabular}{|c|c|}
\hline
$\gamma =0.1$ & $d=140$  \\
     \hline
$T_{\mathrm{init}}(=T,\ell=1)=$ & \\
$d^{1.5} = 1656$  &  $T_{\mathrm{subs}}(=T,\ell \geq 2) =  5$ \\
\hline
$\varepsilon_v $ & $\varepsilon_{\ell, T} $ \\
\hline
$0.01$& $0.05$ for $\ell = 1$, $\frac{0.01}{\sqrt{pT_{\mathrm{subs}}}}$ for $\ell \geq 2$ \\
\hline
\end{tabular}
\vspace{4pt}
    \caption{We describe the parameters for our experiments. Here $\gamma$ is the RBF kernel parameter. $d$ is the dimension of the Rahimi-Recht hash function $\h_1(\cdot)$. We use two different $T$ parameters for different epochs given by $T_{\mathrm{init}}$ (for the first epoch) and $T_{\mathrm{subs}}$ (for subsequent epochs). $\epsilon_{v}$ is the $\epsilon$ parameter for $\h_2(\cdot)$ for the validation set and $\epsilon_{\ell,T}$ is set for $\h_2(\cdot)$ on summaries $D_s$ over epochs $\ell$.}
    \label{table:parameters}
\end{table}

\textbf {MNIST Dataset:}
We now demonstrate similar results on a standard hand-written digit recognition dataset namely MNIST. We start with a brief description of the setup.

{\em Training:} We distribute the MNIST training dataset among five data owners based on digit labels as follows. Splitting the digits into groups $[[0,1],[3,4],[5,6],[7,8],[9,2]$, we allocate the training data corresponding to these digits to the corresponding data owners. {\em Testing:} The test set contains data corresponding to two labels $[3,4]$ sampled with ratio $[0.7,0.3]$. {\em Validation:} We sample (and remove) from the test set with probability $0.25$ to construct the validation dataset.

\begin{figure}[h!]
  \begin{subfigure}[b]{0.46\linewidth}
    \includegraphics[width=7cm]{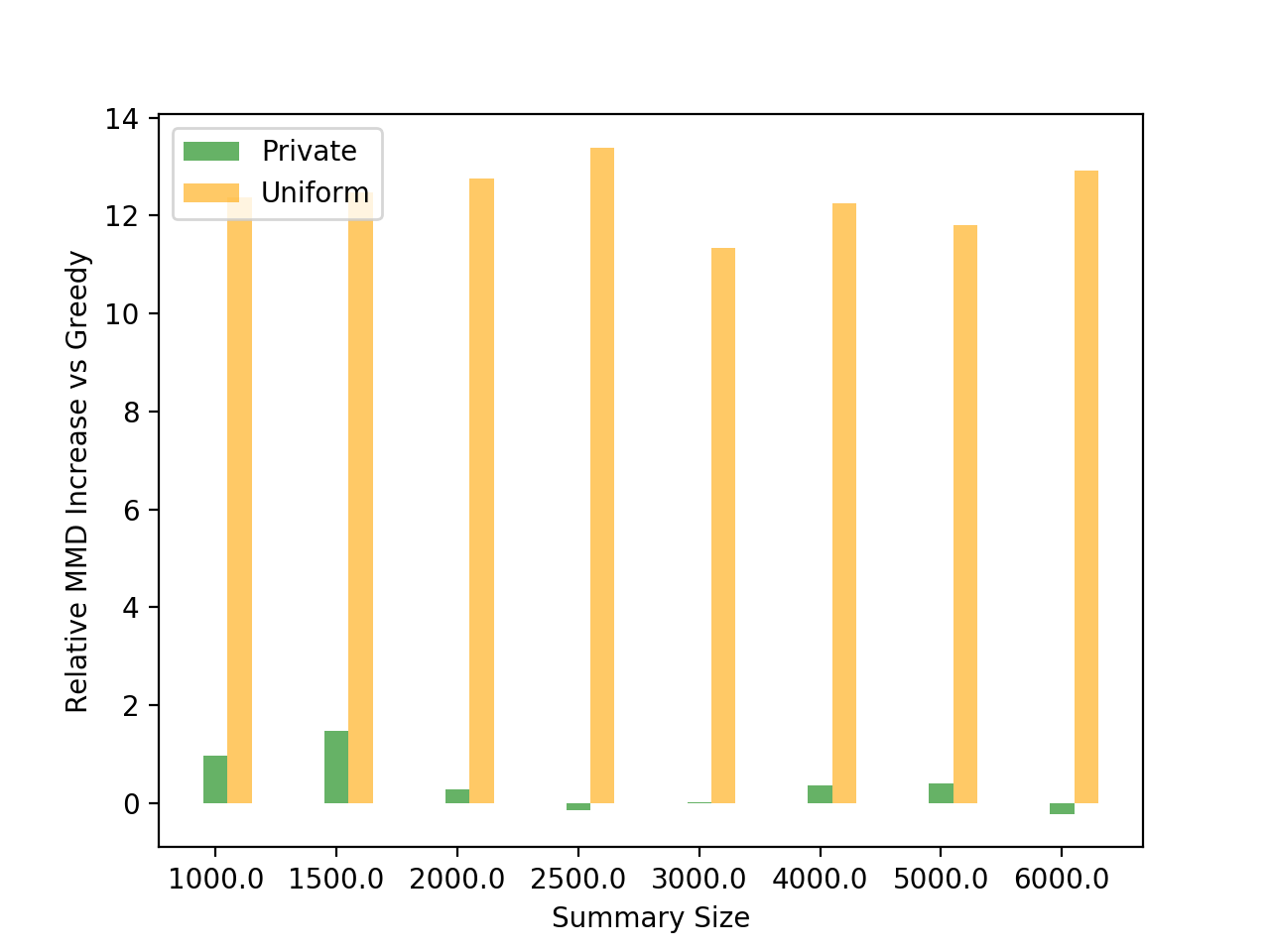}
 
  \end{subfigure}
  \begin{subfigure}[b]{0.46\linewidth}
    \includegraphics[width=7cm]{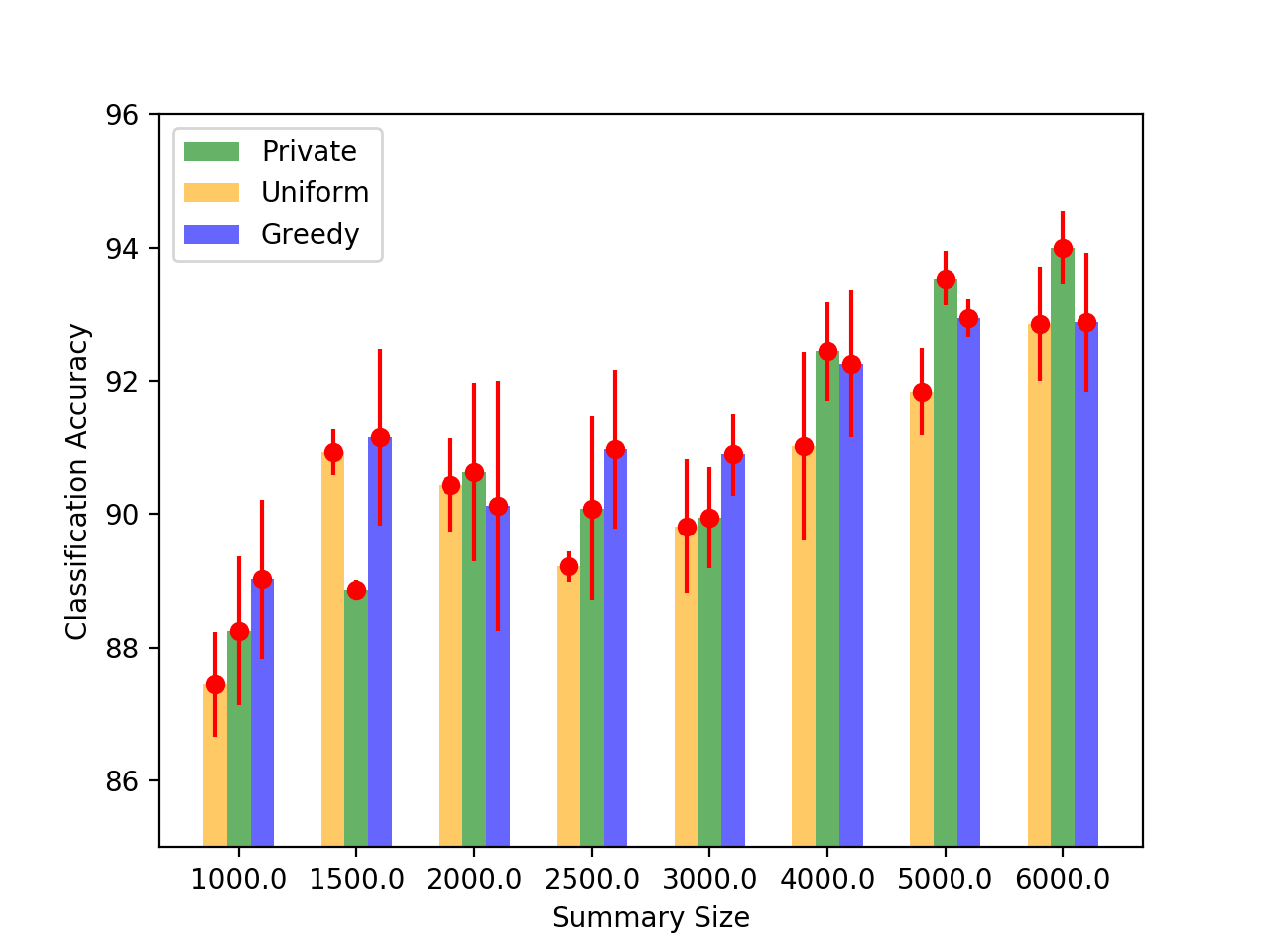}
  \end{subfigure}
  \caption{{\em MNIST Dataset} (Top): Comparison of the percentage increase in $MMD^2$ of both the private and uniform sampling algorithms with respect to baseline greedy algorithm. Lower values indicate better performance. Consistently there is $\mathbf{10}$-$\mathbf{15\%}$ performance difference from uniform sampling. (Bottom): Comparison of the classification accuracy of the three algorithms using a neural network with one hidden layer of $32$ units. Higher numbers indicate better performance.}
  \label{fig:mnist}
\end{figure}


As before, we vary the number of samples and in Figure~ \ref{fig:mnist}, compare the percentage increase in $MMD^2$ with respect to greedy, i.e., $\frac{MMD^2(ALGM) - MMD^2(GREEDY)}{MMD^2(GREEDY)}\times 100$. Recall from above that $ALGM$ is either our private greedy algorithm or the uniform sampling algorithm. Our results show that we consistently outperform the uniform sampling algorithm by at least $10$-$13\%$. In Figure~\ref{fig:mnist}, we compare the performance of these algorithms using a neural net with $32$ neurons  in a single hidden layer and drop out of $0.2$. Note that since our goal is to demonstrate that the relative performance of these algorithms, we are not concerned with the actual performance numbers (prior works on this subject in fact use a much simple 1-Nearest Neighbor  classifier).   We again find that the private algorithm beats uniform sampling in most cases.

\end{document}